%% file: iclr2027_conference.tex
\documentclass{article}

\usepackage[preprint]{neurips_2026}

\usepackage[utf8]{inputenc} 
\usepackage[T1]{fontenc}    
\usepackage{url}            
\usepackage{booktabs}       
\usepackage{amsfonts}       
\usepackage{nicefrac}       
\usepackage{microtype}      
\usepackage{xcolor}         

\usepackage{microtype}
\usepackage{graphicx}
\usepackage{subcaption}

\usepackage[utf8]{inputenc}
\usepackage[T1]{fontenc}
\usepackage[colorlinks,citecolor=orange]{hyperref}
\usepackage{tcolorbox}
\usepackage{wrapfig}
\usepackage{caption}
\usepackage{amsmath}
\usepackage{amssymb}
\usepackage{mathtools}
\usepackage{amsthm}
\usepackage{multirow}
\usepackage{bbm}
\usepackage{algorithm}
\usepackage{algorithmic}
\usepackage{setspace}
\usepackage{graphicx}
\usepackage[capitalize,noabbrev]{cleveref}

\tcbuselibrary{breakable, skins}
\definecolor{myorange}{RGB}{255, 140, 0}

\usepackage{aliascnt}

\theoremstyle{plain}
\newtheorem{theorem}{Theorem}[section]

\newaliascnt{proposition}{theorem}
\newtheorem{proposition}[proposition]{Proposition}
\aliascntresetthe{proposition}

\newaliascnt{assumption}{theorem}
\newtheorem{assumption}[assumption]{Assumption}
\aliascntresetthe{assumption}

\newaliascnt{lemma}{theorem}

\aliascntresetthe{lemma}

\newaliascnt{corollary}{theorem}
\newtheorem{corollary}[corollary]{Corollary}
\aliascntresetthe{corollary}

\theoremstyle{remark}
\newaliascnt{remark}{theorem}
\newtheorem{remark}[remark]{Remark}
\aliascntresetthe{remark}

\theoremstyle{definition}
\newtheorem{definition}{Definition}

\crefname{theorem}{Theorem}{Theorems}
\Crefname{theorem}{Theorem}{Theorems}
\crefname{proposition}{Proposition}{Propositions}
\Crefname{proposition}{Proposition}{Propositions}
\crefname{assumption}{Assumption}{Assumptions}
\Crefname{assumption}{Assumption}{Assumptions}
\crefname{lemma}{Lemma}{Lemmas}
\Crefname{lemma}{Lemma}{Lemmas}
\crefname{remark}{Remark}{Remarks}
\Crefname{remark}{Remark}{Remarks}
\crefname{corollary}{Corollary}{Corollaries}
\Crefname{corollary}{Corollary}{Corollaries}
\crefname{definition}{Definition}{Definitions}
\Crefname{definition}{Definition}{Definitions}

\newcommand{\sotabest}[1]{\textcolor{myorange}{$\mathbf{#1}$}}
\newcommand{\sotasecond}[1]{$\underline{#1}$}

\title{NFTR: From Provable Mode-Averaging to Geodesic Subgoal Selection in Offline Goal-Conditioned RL}

\author{
\textbf{Erdemt Bao}$^{1,*}$ \quad
\textbf{Xing Lei}$^{2,*\,\dagger}$ \quad
\textbf{Jun Chen}$^{3}$ \\
\textbf{Xuetao Zhang}$^{2}$ \quad
\textbf{Donglin Wang}$^{4}$ \\[2pt]
$^{1}$Huazhong University of Science and Technology \quad
$^{2}$Xi'an Jiaotong University \\
$^{3}$University of Electronic Science and Technology of China \quad
$^{4}$Westlake University \\[2pt]
\texttt{baoerdemt366@gmail.com} \quad
\texttt{leixing@stu.xjtu.edu.cn} \quad
\texttt{junchen@std.uestc.edu.cn} \\
\texttt{xuetaozh@xjtu.edu.cn} \quad
\texttt{wangdonglin@westlake.edu.cn}
}

\begin{document}

\maketitle

\begingroup
\renewcommand{\thefootnote}{}
\footnotetext{* Equal contribution. \hspace{1.5em} $\dagger$ Corresponding author.}
\addtocounter{footnote}{-1}
\endgroup

\begin{abstract}
Hierarchical Implicit Q-Learning (HIQL), an offline goal-conditioned RL method, selects subgoals by value-function advantages alone. This rule has two coupled failure modes. \emph{Optimistic bias} treats lucky stochastic outcomes as skillful choices, and \emph{mode collapse} reduces a multi-modal subgoal distribution to a single Gaussian mean that often falls in unreachable regions. We propose \textbf{NFTR} (\textbf{N}ormalizing \textbf{F}low subgoal policies with \textbf{T}riangle-slack \textbf{R}eweighting). A conditional Normalizing Flow replaces the Gaussian policy. A closed-form mode-averaging result identifies the Gaussian limitation, while conditional flows support multi-modal weighted maximum likelihood and direct sampling. A triangle slack score, computed from a jointly trained MRN distance head whose triangle inequality is guaranteed architecturally, corrects the AWR weight multiplicatively and measures a detour inside the learned geometry without requiring exact distance recovery. Triangle-slack vanishes on geodesics in deterministic MDPs and remains a conservative upper bound on composability violation under stochastic dynamics. The RWDR objective preserves AWR's population-level monotonic improvement and admits a three-term suboptimality decomposition. On OGBench the flow carries the larger share of the gain, while the full co-trained configuration adds task-dependent gains. The combined method avoids Gaussian mode averaging and performs well on stochastic tasks. GitHub page:
\href{https://github.com/erdemtbao/NFTR}
{\textcolor{orange}{\texttt{https://github.com/erdemtbao/NFTR}}}
\end{abstract}


\section{Introduction} \label{sec:introduction}

Offline goal-conditioned reinforcement learning (GCRL) \citep{ogbench_park2025} aims to learn policies that can reach arbitrary goals from static datasets without environment interaction. This setting is particularly valuable for robotics and autonomous systems where online data collection is expensive or dangerous. Among recent advances, Hierarchical Implicit Q-Learning (HIQL)~\citep{park2023hiql} has achieved state-of-the-art performance by decomposing goal-reaching into high-level subgoal selection and low-level action execution, using a shared value function to guide both levels.

Two coupled failures of HIQL motivate this work. First, because the high-level subgoal is selected purely by the learned value function, stochastic environments produce \textbf{optimistic bias}. A trajectory reaching the goal through favorable randomness is treated identically to one succeeding through repeatable reachability, so off-path subgoals receive inflated AWR weights. Second, because the high-level policy is a unimodal Gaussian, multi-modal subgoal distributions induced by branching corridors or stochastic outcomes trigger \textbf{mode collapse}. The Gaussian mean is regressed toward a convex combination of mode centers, which often falls inside an unreachable region rather than on any valid path. Together these failures raise a single question. \textit{Can we learn subgoals that are simultaneously high-value, reliably reachable, and faithful to multi-modal path structure?}

We answer this question with \textbf{NFTR} (\textbf{N}ormalizing \textbf{F}low subgoal policies with \textbf{T}riangle-slack \textbf{R}eweighting). A conditional Normalizing Flow replaces HIQL's Gaussian high-level policy, and a triangle-slack score derived from an architecturally inequality-preserving quasimetric reweights the AWR objective, so geometric inconsistency acts as a mechanistic proxy for unreliable lucky transitions. We order our contributions by empirical weight. A closed-form diagnosis makes an expressive high-level class a consequence of the objective. A placement result puts the flow at the high level alone. Triangle-slack reweighting acts on the AWR weight, not on the value function. Two concurrent methods also insert flows into HIQL at both levels, and \cref{app:flow_hiql} compares against them. On OGBench \citep{ogbench_park2025}, large gains occur on teleport navigation and PointMaze medium stitching with multi-modal subgoals.

\section{Preliminaries}
\label{sec:background}
\subsection{Offline Goal-Conditioned Reinforcement Learning}
\label{sec:bg_gcrl}

We consider a controlled Markov process $\mathcal{M} = (\mathcal{S}, \mathcal{A}, P, \gamma, \mu)$ with state space $\mathcal{S}$, action space $\mathcal{A}$, transition dynamics $P(s'|s,a)$, discount factor $\gamma \in (0,1)$, and initial state distribution $\mu(s_0)$. A goal-conditioned policy $\pi(a|s,g)$ aims to reach goal state $g \in \mathcal{S}$ from current state $s$. The goal-conditioned value function $V^\pi(s,g)$ is the expected discount accumulated until $g$ is first reached from $s$ under policy $\pi$:
\begin{equation}
V^\pi(s,g) = \mathbb{E}_\pi\left[\gamma^{T_g} \mid s_0 = s\right],
\qquad
T_g = \inf\{t \geq 0 : s_t = g\},
\label{eq:value_function}
\end{equation}
with $\gamma^{\infty} = 0$. This gives $V^\pi(g,g) = 1$, so $d^{*}(s,g) = -\log V^{*}(s,g)$ satisfies $d^{*}(g,g) = 0$ as \cref{def:quasimetric} requires.
In offline GCRL, we have access only to a static dataset $\mathcal{D} = \{(s_t, a_t, s_{t+1})\}_{t=1}^N$ collected by an unknown behavior policy $\beta$, with goals typically sampled from future states via hindsight relabeling~\citep{andrychowicz2017hindsight}. The challenge is to learn a policy $\pi$ that maximizes goal-reaching performance without additional environment interaction, requiring careful handling of distribution shift between the learned policy and the behavior policy~\citep{levine2020offline}.

\subsection{HIQL: Hierarchical Implicit Q-Learning}
\label{sec:bg_hiql}
HIQL~\citep{park2023hiql} factorizes goal-reaching into a high-level policy $\pi^H(z\mid s,g)$ that proposes the latent code $z=\phi([s;w])$ of a $k$-step waypoint $w$, and a low-level policy $\pi^L(a\mid s,z)$ that drives the agent toward it. The encoder $\phi:\mathcal{S}\times\mathcal{S}\to\mathbb{R}^d$ is an intermediate layer of the parameterized value function $V(s,\phi([s;g]))$, learned end-to-end with no separate representation objective; HIQL's Proposition~5.1 shows this representation is sufficient for optimal control in deterministic MDPs.

The shared value function is fitted with the action-free expectile-regression variant of IQL~\citep{kostrikov2022offline,park2023hiql}:
\begin{equation}
\mathcal{L}_V \;=\; \mathbb{E}_{(s,s')\sim\mathcal{D},\,g\sim p(g\mid\tau)}\left[L_\tau^2\bigl(r(s,g)+\gamma V^{\mathrm{tgt}}(s',g)-V_\theta(s,g)\bigr)\right],
\label{eq:value_loss}
\end{equation}
with $L_\tau^2(u)=|\tau-\mathbf{1}(u<0)|\cdot u^2$ for $\tau\in(0.5,1)$. The high-level policy is then extracted by advantage-weighted regression~\citep{peng2019advantage}:
\begin{equation}\label{eq:hiql_high}
\mathcal{L}_{\pi^H} \;=\; -\,\mathbb{E}_{(s,w,g)\sim\mathcal{D}}\left[\exp\bigl(\alpha_H A_{\mathrm{env}}(s,w,g)\bigr)\cdot\log\pi^H\bigl(\phi([s;w])\,\bigm|\,s,g\bigr)\right],
\end{equation}
where $A_{\mathrm{env}}(s,w,g)=V_\theta(w,g)-V_\theta(s,g)$ and $\alpha_H>0$. Following HIQL's Section~5.2, the regression target is the latent $z=\phi([s;w])$ rather than raw $w$, a choice required for image-based and consistently helpful for state-based settings; NFTR keeps it unchanged.

\textbf{Two technical hooks for our analysis.} HIQL's hierarchy is justified by a signal-to-noise argument~\citep[Proposition~4.1]{park2023hiql} that is silent on the high-level policy class, leaving the unimodal Gaussian inherited from POR~\citep{xu2022policy} unanalyzed. The action-free objective in \cref{eq:value_loss} is also unbiased only under deterministic dynamics~\citep[Section~3]{park2023hiql}, and AWR in \cref{eq:hiql_high} exponentially amplifies any resulting value overestimation. \cref{sec:nf_policy,sec:rwdr} exploit these two openings.

\subsection{Triangle Inequality as a Structural Primitive}
\label{sec:distance}
We single out one structural property, the triangle inequality, from the broader machinery of quasimetric learning. Earlier work~\citep{wang2023optimal,liu2023metric,myers2024learning,myers2025offline} treats the quasimetric as an \emph{estimation target}, where downstream performance scales with how accurately the distance is recovered. We instead treat the inequality alone as a \emph{structural prior} on the AWR weight, regardless of whether the underlying distance is recovered accurately.

\begin{definition}[Quasimetric]
\label{def:quasimetric}
A function $d: \mathcal{S} \times \mathcal{S} \to \mathbb{R}_{\geq 0}$ is a \textbf{quasimetric} if it satisfies: (1) $d(s, s) = 0$, (2) $d(s, g) > 0$ for $s \neq g$, and (3) $d(s, g) \leq d(s, w) + d(w, g)$ for all $s, w, g$. Unlike metrics, quasimetrics need not be symmetric.
\end{definition}
In what follows, $d_\theta$ stands for a distance whose triangle inequality is guaranteed at the architectural level. The concrete MRN parameterization and its training losses are deferred to \cref{sec:implementation,sec:experiments}. Two consequences of this construction matter for our framework. First, the inequality holds \emph{architecturally}, so the slack is non-negative before any training. Second, distance training shapes a \emph{geometry inside which detours can be measured}, without requiring exact recovery of reachability distances. The first property is exact, and \cref{app:slack_evidence} measures the second empirically.
\subsection{Normalizing Flows (NFs)}
\label{sec:bg_nf}
NFs~\citep{dinh2014nice,dinh2017density,papamakarios2021normalizing,ghugare2025normalizing} define probability distributions by transforming simple base distributions through invertible functions, enabling both exact density computation and efficient sampling.
\begin{definition}[Normalizing Flow]
A Normalizing Flow defines a distribution $p_\theta(z)$ over $z \in \mathbb{R}^d$ through an invertible transformation $f_\theta: \mathbb{R}^d \to \mathbb{R}^d$ applied to a base distribution $p_{\text{base}}(\epsilon)$:
\begin{equation}
z = f_\theta^{-1}(\epsilon), \quad \epsilon \sim p_{\text{base}}(\epsilon).
\end{equation}
The density is computed via the change-of-variables formula:
\begin{equation}
\log p_\theta(z) = \log p_{\text{base}}(f_\theta(z)) + \log \left|\det \frac{\partial f_\theta}{\partial z}\right|.
\label{eq:nf_density}
\end{equation}
\end{definition}

A conditional flow conditions the transformation $f_\theta$ on context $c$ to model $p_\theta(z\mid c)$, which we instantiate with RealNVP~\citep{dinh2017density} affine coupling layers in \cref{sec:implementation}.


\section{Method: NFTR}
\label{sec:method}

We present NFTR, which addresses HIQL's limitations through two complementary innovations: a Normalizing Flow high-level subgoal policy for multi-modal subgoal distributions (\cref{sec:nf_policy}), and triangle-slack composability reweighting for filtering geometrically inconsistent subgoals (\cref{sec:rwdr}). We then provide theoretical analysis of the combined method (\cref{sec:theory}) and describe implementation details (\cref{sec:implementation}). The overall architecture is illustrated in \cref{fig:overall}.

\begin{figure*}[t]
\centering
\includegraphics[width=\textwidth]{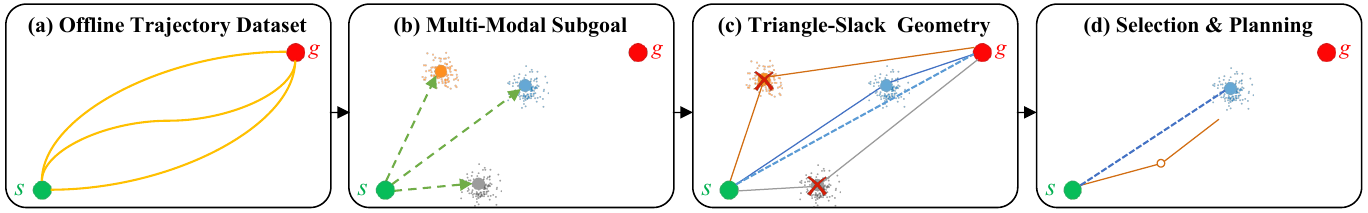}

\caption{\textbf{Overview of NFTR.} (a) Offline trajectories. (b) A Normalizing Flow high-level policy models a multi-modal subgoal distribution. (c) A quasimetric $d_\theta$ induces the slack $\Delta(s,w,g)=\mathrm{ReLU}(d(s,w){+}d(w,g){-}d(s,g))$, which downweights non-composable subgoals. (d) Execution yields a feasible stitched path. RWDR combines $A_{\text{env}}$ from $V_\theta$ and $\Delta$ from $d_\theta$ to train the policy.}

\label{fig:overall}
    \vspace{-1.6em}
\end{figure*}

\subsection{Normalizing Flow High-Level Subgoal Policy}
\label{sec:nf_policy}
Two facts about HIQL's high-level training set drive the analysis. The latent code $z=\phi([s;w])$ is the deterministic image of the $k$-step waypoint $w=s_{t+k}$~\citep[Algorithm~1]{park2023hiql}, and for a fixed $(s,g)$ appearing in many trajectories the empirical distribution of $w$ is shaped by trajectory-level path multiplicity rather than single-step stochasticity. Branching corridors, teleporters that randomize landing sites, and observation noise that spreads the same physical waypoint across $\phi$-space all produce disjoint waypoint clusters separated by unreachable regions $\mathcal{S}_{\mathrm{wall}}\subset\mathbb{R}^d$. We model this as a state-dependent mixture
\begin{equation}
p_{\mathcal{D}}(z\mid s,g) \;=\; \sum_{m=1}^{M(s,g)} \pi_m(s,g)\,q_m(z\mid s,g),
\label{eq:dataset_mixture}
\end{equation}
with path-connected, pairwise disjoint component supports $\mathrm{supp}(q_m)$. We treat $\mathcal{S}_{\mathrm{wall}}$ as the empirical $\phi$-image of the physical unreachable region observed in our visualizations (\cref{fig:loglik_heatmap_task1}), since HIQL's representation-sufficiency~\citep[Proposition~5.1]{park2023hiql} is only exact in the deterministic case.

The next theorem characterizes the regime in which HIQL's POR-inherited unimodal Gaussian fails. To our knowledge it is the first closed-form characterization of HIQL's mode-collapse limitation.

\begin{theorem}[Mode-Averaging of Gaussian High-Level Subgoal Policies]
\label{thm:mode_averaging}
Let $\pi^H_\theta(z\mid s,g)=\mathcal{N}(z;\,\mu_\theta(s,g),\,\sigma^2 I_d)$ with $\sigma>0$ fixed and $\mu_\theta$ unconstrained, trained by HIQL's AWR objective
\begin{equation}
\mathcal{L}^{\mathrm{AWR}}(\theta) \;=\; -\,\mathbb{E}_{(s,w,g)\sim\mathcal{D}}\left[e^{\alpha_H A^{H}(s,w,g)}\,\log\mathcal{N}\bigl(\phi([s;w]);\,\mu_\theta(s,g),\,\sigma^2 I_d\bigr)\right],
\label{eq:awr_gaussian}
\end{equation}
with $\alpha_H>0$ and $e^{\alpha_H A^{H}}$ bounded above and integrable under $p_{\mathcal{D}}$ (enforced by the weight clipping in \cref{alg:training}). Let \cref{eq:dataset_mixture} hold with $M(s,g)\geq 2$ and write $w(z)$ for the deterministic preimage of $z=\phi([s;w])$. Define $p^{\star}_{s,g}(z)\coloneqq Z^{-1}_{s,g}e^{\alpha_H A^{H}(s,w(z),g)}p_{\mathcal{D}}(z\mid s,g)$ with normalizer $Z_{s,g}=\mathbb{E}_{p_{\mathcal{D}}}[e^{\alpha_H A^{H}}\mid s,g]$. The unique minimizer of \cref{eq:awr_gaussian} in $\mu_\theta(s,g)$ is
\begin{equation}
\mu^{\star}_\theta(s,g) \;=\; \mathbb{E}_{z\sim p^{\star}_{s,g}}[z] \;=\; \sum_{m=1}^{M(s,g)} \tilde{\pi}_m(s,g)\,\tilde{\mu}_m(s,g),
\label{eq:gaussian_closed_form}
\end{equation}
where $\tilde{\mu}_m$ and $\tilde{\pi}_m$ are the AWR-tilted modal centers and weights (full forms in \cref{app:proof_mode_averaging}). Hence $\mu^{\star}_\theta(s,g)\in\mathrm{conv}\{\tilde{\mu}_m(s,g)\}$. When the tilted weights produce $\mu^{\star}_\theta(s,g)\in\mathcal{S}_{\mathrm{wall}}$, generic for approximately balanced modes separated by the wall, the Gaussian places mass on $\mathcal{S}_{\mathrm{wall}}$ bounded below by a strictly positive constant for every $\sigma$ in any fixed compact range.
\end{theorem}

\begin{wrapfigure}{r}{0.58\textwidth}
  \vspace{-1.0em}
  \centering
  \includegraphics[width=\linewidth]{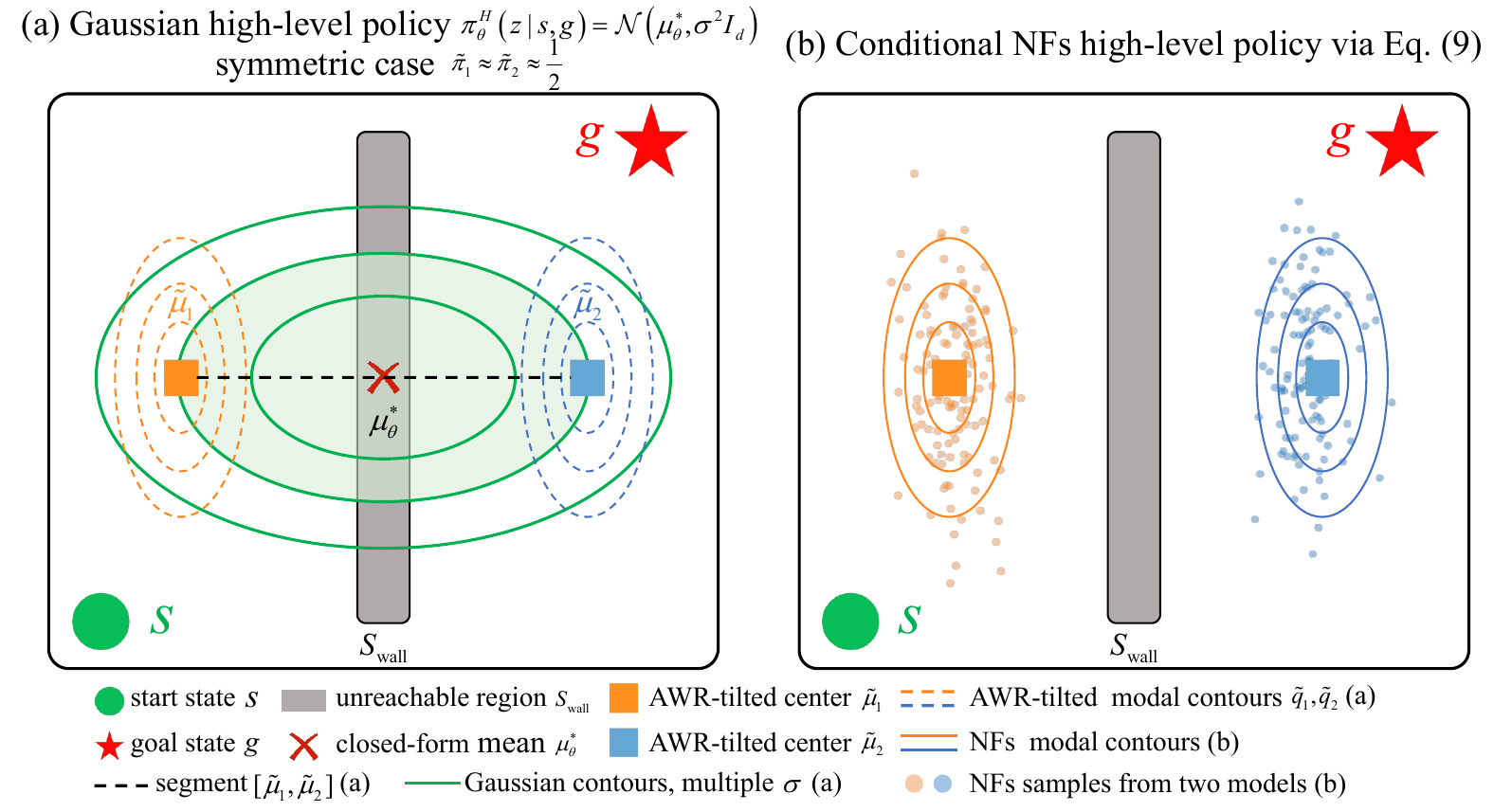}
  \vspace{-0.5em}
  \caption{\textbf{Geometric illustration of \cref{thm:mode_averaging}.} \textbf{(a)} In the symmetric two-mode case $\widetilde{\pi}_1\approx\widetilde{\pi}_2\approx\tfrac{1}{2}$, the AWR-optimal Gaussian mean $\mu^{*}_\theta$ from \cref{eq:gaussian_closed_form} sits at the midpoint of $[\widetilde{\mu}_1,\widetilde{\mu}_2]$ inside $\mathcal{S}_{\mathrm{wall}}$, and the Gaussian (green ellipses at multiple $\sigma$) leaks irreducible mass onto $\mathcal{S}_{\mathrm{wall}}$. \textbf{(b)} A conditional NF (\cref{eq:nf_logdensity}) places mass on both corridors and drives the wall mass arbitrarily small as capacity grows~\citep{papamakarios2021normalizing}.}
  \label{fig:nf_vs_gaussian}
  \vspace{-1.2em}
\end{wrapfigure}
\cref{thm:mode_averaging} states that the AWR-optimal Gaussian mean is a fixed convex combination of AWR-tilted modal centers, falls inside any region those centers straddle, and irreducibly leaks mass into that region; \cref{fig:nf_vs_gaussian} illustrates this failure on a two-corridor case alongside the corresponding NF behavior. The full proof is in \cref{app:proof_mode_averaging}.
\begin{remark}[The diagnosis does not depend on the density]
\label{rem:regression_invariance}
\cref{eq:awr_gaussian} is a weighted maximum likelihood objective, but its minimizer in $\mu_\theta$ is unchanged when the Gaussian head is fitted by weighted least-squares regression. Step~1 of \cref{app:proof_mode_averaging} reduces \cref{eq:awr_gaussian} to $\tfrac{1}{2\sigma^{2}}\,\mathbb{E}[e^{\alpha_H A^{H}}\|z-\mu_\theta\|_2^2]$ up to a term free of $\mu_\theta$. The convex combination in \cref{eq:gaussian_closed_form} and the wall-mass conclusion therefore hold under either route, so the scoping below concerns the admissible model class rather than the failure described here.
\end{remark}

\textbf{Policy choice.} The mode-averaging result motivates a policy that represents multiple subgoal modes (R1). Weighted maximum likelihood calls for an exact log density (R2). We also prefer direct sampling at execution time (R3). A conditional NF meets these requirements:
\begin{equation}
\log\pi^H_\theta(z\mid s,g) \;=\; \log\mathcal{N}\bigl(f_\theta(z;\,s,g);\,0,I_d\bigr) \;+\; \log\left|\det\frac{\partial f_\theta(z;\,s,g)}{\partial z}\right|
\label{eq:nf_logdensity}
\end{equation}
This parameterization provides an exact density (R2), represents multiple modes (R1)~\citep{papamakarios2021normalizing}, and permits direct sampling through $f_\theta^{-1}$ (R3). The NF is therefore a practical fit for this objective. Weighted regression can train other expressive generators without exact density evaluation. \cref{app:nf_vs_other,app:flow_hiql} examines those alternatives. \cref{fig:nf_vs_gaussian} illustrates the two-corridor case, and \cref{fig:loglik_heatmap_task1} (\cref{app:additional}) shows learned policy densities on \texttt{antmaze-teleport-navigate}. Flow Q-Learning (FQL)~\citep{park2025flow} and NF-RLBC~\citep{ghugare2025normalizing} parameterize action policies rather than high-level subgoals.
\subsection{Triangle-Slack Composability Reweighting}
\label{sec:rwdr}

Multi-modal modeling alone does not address optimistic bias. The AWR advantage $A_{\text{env}}(s,w,g) = V(w,g) - V(s,g)$ tells us whether $w$ \emph{appears} closer to $g$ under the learned value function, but it cannot separate two qualitatively different events that both produce a positive advantage in the dataset. The first is a subgoal reached along a reliable path. The second is a subgoal reached via a lucky transition (a teleporter outcome, an observation-noise realization). Both inflate $A_{\text{env}}$, and AWR amplifies them identically through $\exp(\alpha_H A_{\text{env}})$. The root cause is that $V$ answers a sample-level question, namely ``did this subgoal lead to success in this data?'' rather than the population-level question we actually care about, namely ``is this subgoal reliably on the typical path from $s$ to $g$?''

\begin{wrapfigure}[24]{r}{0.45\linewidth}
    \vspace{-2.0em}
    \centering
	\centerline{\includegraphics[width=0.4\textwidth]{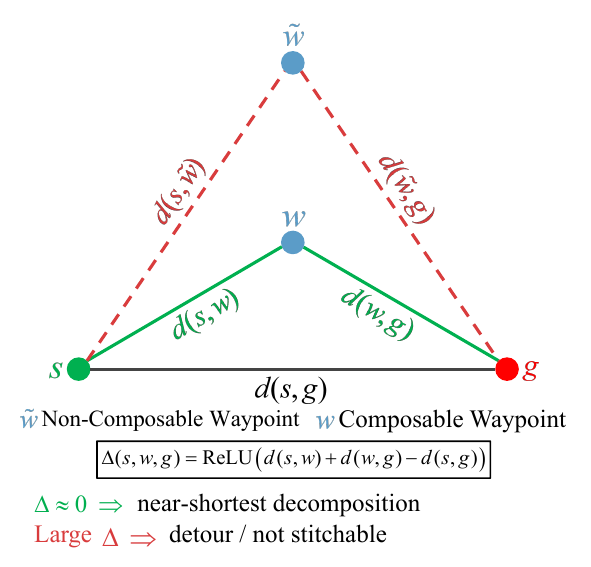}}
    \caption{\textbf{Triangle-slack geometry.} A composable waypoint $w$ approximately satisfies the triangle equality, i.e., $d(s,w)+d(w,g)\approx d(s,g)$, yielding a small triangle slack $\Delta(s,w,g)\approx 0$ (green edges). In contrast, a non-composable waypoint $\tilde{w}$ violates this relation, producing a large slack $\Delta(s,\tilde{w},g)>0$ (red dashed edges), which indicates a detour / poor composability and should be downweighted by RWDR.}
    \label{fig:triangle_slack}
    \vspace{-1.6em}
\end{wrapfigure}
\textbf{Triangle-slack as a geometric consistency signal.} We introduce triangle-slack as a consistency signal that identifies composable subgoals independently of value estimates. Given a learned quasimetric distance $d_\theta$ (detailed in \cref{sec:implementation}), we define the triangle-slack of subgoal $w$ for state-goal pair $(s,g)$ as
\begin{equation}
\Delta(s, w, g) \;=\; \max\big(0,\; d_\theta(s, w) + d_\theta(w, g) - d_\theta(s, g)\big).
\label{eq:triangle_slack}
\end{equation}
By the triangle inequality property of quasimetrics, $\Delta \geq 0$ always holds, with $\Delta \approx 0$ when $w$ lies on a geodesic path from $s$ to $g$ and $\Delta > 0$ when reaching $g$ via $w$ requires a detour compared to the direct path. \cref{fig:triangle_slack} illustrates this geometry: composable waypoints incur near-zero slack while non-composable ones produce positive slack indicating a detour.

We now make these geometric properties precise. We first formalize the notion of a geodesic subgoal:
\begin{definition}[Geodesic Subgoal]
\label{def:geodesic}
A subgoal $w$ is \textbf{geodesic} for $(s, g)$ under quasimetric $d$ if $d(s, g) = d(s, w) + d(w, g)$.
\end{definition}
The following proposition characterizes the four key properties of triangle-slack as a composability signal.
\begin{proposition}[Slack Characterization]
\label{prop:slack}
For any quasimetric $d$ satisfying \cref{def:quasimetric} and any subgoal $w \in \mathcal{S}$:
\begin{enumerate}
    \item $\Delta(s, w, g) \geq 0$ for all $s, w, g \in \mathcal{S}$.
    \item $\Delta(s, w, g) = 0$ if and only if $w$ is geodesic for $(s, g)$.
    \item $\Delta(s, w, g) > 0$ if and only if reaching $g$ via $w$ requires a strict detour, i.e., $d(s,w) + d(w,g) > d(s,g)$.
    \item In a deterministic MDP, under the optimal distance $d^*(s,g) = -\log V^*(s,g)$, geodesic subgoals are exactly those on optimal paths. In a general (stochastic) MDP, the multiplicative value decomposition relaxes to $V^*(s,g) \geq V^*(s,w) \cdot V^*(w,g)$, so $d^*(s,g) \leq d^*(s,w) + d^*(w,g)$ becomes a (strict) inequality for off-path $w$, and triangle-slack acts as a \emph{conservative upper bound} on composability violation.
\end{enumerate}
\end{proposition}

The proof is in \cref{app:proof_slack}. Parts (1) to (3) say triangle-slack vanishes exactly on geodesics, and part (4) says it remains a conservative upper bound on composability violation under stochastic dynamics. The link to lucky-transition filtering rests on two properties. Non-negativity is architectural, so the correction can only remove weight and never reward a shortcut. A larger detour inside the learned geometry of $d_\theta$ gives a larger slack, and \cref{app:slack_evidence} measures its association with unreliable transitions empirically.

\textbf{The RWDR objective.} Building on \cref{eq:triangle_slack}, we combine the value advantage and triangle-slack penalty into the RWDR weighting scheme:
\begin{equation}
\log w(s, w, g) \;=\; \alpha_H \cdot A_{\text{env}}(s, w, g) - \kappa \cdot \Delta(s, w, g),
\label{eq:rwdr_weight}
\end{equation}
where $\kappa\geq 0$ controls the geometric penalty. Setting $\kappa=0$ recovers the value-only weight. For $\kappa>0$, the weight $w(s,w,g)$ is high when the subgoal has both high value advantage (making progress toward the goal) and low triangle-slack (geometrically consistent with being on-path). When the two signals conflict, RWDR can downweight a high-value subgoal with large slack. The final high-level policy loss becomes:
\begin{equation}
\mathcal{L}_{\pi^H} \;=\; -\mathbb{E}_{(s,w,g)\sim\mathcal{D}}\left[w(s, w, g) \cdot \log \pi^H_{\text{NF}}(\phi([s;w]) \mid s, g)\right],
\label{eq:high_loss}
\end{equation}
where we apply stop-gradient to both $V_\theta$ and $d_\theta$ when computing weights, treating them as fixed scorers for each policy update. This prevents gradients from the policy loss from updating the value and distance heads. It does not freeze those heads or the value encoder $\phi$ across training iterations; $\phi$ continues to evolve under the value loss, so the NF regression targets can drift during training.

The reweighting in \cref{eq:rwdr_weight} can be equivalently expressed as a multiplicative correction to the AWR weights, which we record as a corollary for later use:

\begin{assumption}[Geometric Consistency]
\label{ass:geometric}
The distance $d_\theta$ satisfies the triangle inequality exactly:
\begin{equation}
d_\theta(s,g) \leq d_\theta(s,w) + d_\theta(w,g), \quad \forall s,w,g \in \mathcal{S}.
\end{equation}
This property is guaranteed by the MRN architecture and holds regardless of training.
\end{assumption}

\begin{corollary}[Geometric Regularization Effect]
\label{cor:geometric_reg}
Under \cref{ass:geometric}, for any distance function $d_\theta$ satisfying the triangle inequality and $\kappa\geq 0$, the raw RWDR weight is no greater than the value-only weight:
\begin{equation}
w_{\text{RWDR}}(s,w,g) \;=\; w_{\text{HIQL}}(s,w,g) \cdot \exp\big(-\kappa\, \Delta(s,w,g)\big).
\end{equation}
If $\kappa>0$ and $\Delta(s,w,g)>0$, then $w_{\text{RWDR}}(s,w,g)<w_{\text{HIQL}}(s,w,g)$ before clipping and batch normalization.
\end{corollary}

The proof is provided in \cref{app:proof_geometric_reg}. For $\kappa>0$, RWDR exponentially downweights positive-slack subgoals in the raw weight. The local properties established here describe the geometric correction. The system-level consequences of combining the NF high-level policy from \cref{sec:nf_policy} with the RWDR objective in \cref{eq:high_loss} are analyzed in the next section.

\subsection{Theoretical Analysis of NFTR}
\label{sec:theory}
\cref{sec:nf_policy,sec:rwdr} established two component-level results: NF parameterization avoids mode-averaging at the high-level policy, and triangle-slack RWDR functions as a geometric composability filter (\cref{prop:slack,cor:geometric_reg}). We now combine both into system-level guarantees, under two regularity assumptions.

\begin{assumption}[Single-Policy Concentrability]
\label{ass:coverage}
There exists $C_\pi<\infty$ such that $d^{\pi^{H,*}}(s,w)/d^\beta(s,w)\leq C_\pi$ for all $(s,w)$ in the support of the optimal high-level policy $\pi^{H,*}$, where $d^\pi$ is the state-subgoal visitation under $\pi$ and $d^\beta$ is the dataset distribution.
\end{assumption}
\begin{assumption}[Value Estimation Quality]
\label{ass:value}
For high-level samples $(s,w,g)\sim\mathcal{D}$, the learned value satisfies $\mathbb{E}[(V_\theta(s,g)-V^*(s,g))^2]\leq\epsilon_V$ and $\mathbb{E}[(V_\theta(w,g)-V^*(w,g))^2]\leq\epsilon_V$.
\end{assumption}
Combined with \cref{ass:geometric}, these yield the following decomposition.

\begin{theorem}[RWDR Suboptimality Decomposition]
\label{thm:rwdr}
Assume $|\hat A_{\mathrm{env}}|,|A^*_{\mathrm{env}}|\leq A_{\max}$ and $d_\theta\leq D_{\max}$ (so $\Delta\leq 2D_{\max}$). Under \cref{ass:coverage,ass:geometric,ass:value}, the population minimizer $\hat{\pi}^H$ of the empirical RWDR objective (\cref{eq:high_loss}) with $n$ samples satisfies, with high probability,
\begin{equation}
J(\pi^{H,*}) - J(\hat{\pi}^H) \;\leq\; \underbrace{\mathcal{O}(\sqrt{\epsilon_V})}_{\text{value error}} \;+\; \underbrace{\mathcal{O}\bigl(\sqrt{C_\pi\,\mathcal{C}(\Pi)/n}\bigr)}_{\text{sampling error}} \;+\; \underbrace{R_\kappa(d_\theta)}_{\text{regularization}},
\end{equation}
where $R_\kappa(d_\theta)\coloneqq\kappa A_{\max}\,\mathbb{E}_{(s,w,g)\sim\pi^{H,*}}[\Delta(s,w,g)]\leq 2\kappa A_{\max}D_{\max}$ is the residual bias of the geometric reweighting on optimal subgoals, and $\mathcal{C}(\Pi)$ is the policy-class complexity ($\log|\Pi|$ for finite $\Pi$, replaced by Rademacher or covering-number bounds for parameterizations such as RealNVP). The hidden constant in the value-error term may depend on the fixed $\alpha_H$, $\kappa$, $A_{\max}$ and $D_{\max}$.
\end{theorem}
The proof is in \cref{app:proof_rwdr}. The three terms are IQL value-estimation error, the standard $\mathcal{O}(1/\sqrt{n})$ sampling error, and a geometric regularization that vanishes when $d_\theta$ aligns with optimal distances (since $\Delta=0$ on geodesics by \cref{prop:slack}\,(2)) and is otherwise uniformly bounded. Together with \cref{thm:mode_averaging}, this characterizes how NFTR's two components combine: the NF makes $\hat\pi^H$ realizable, while RWDR controls which subgoals dominate the minimization via $R_\kappa(d_\theta)$.

\begin{proposition}[Population-Level Monotonic Improvement]
\label{prop:monotonic}
Assuming exact weighted-MLE optimization over $\Pi$, the RWDR objective inherits AWR's monotonic improvement: for any $\pi^H_{\mathrm{old}}\in\Pi$ and the exact weighted-MLE solution $\pi^H_{\mathrm{new}}$,
\begin{equation}
\mathrm{KL}(\pi^{H,*}_{\mathcal{D}}\,\|\,\pi^H_{\mathrm{new}}) \;\leq\; \mathrm{KL}(\pi^{H,*}_{\mathcal{D}}\,\|\,\pi^H_{\mathrm{old}}),
\end{equation}
where $\pi^{H,*}_{\mathcal{D}}$ is the RWDR-weighted in-distribution optimum.
\end{proposition}
The proof is in \cref{app:proof_monotonic}. Like all AWR-based methods (HIQL, SAW, OTA), \cref{prop:monotonic} does not transfer verbatim to SGD on a non-convex RealNVP parameterization, a gap that originates from $\pi^H$ rather than from RWDR; the evaluation success curves in \cref{app:training-curves} illustrate practical stability but do not establish per-step loss or return improvement.

\subsection{NFTR Algorithm}
\label{sec:implementation}
NFTR alternates per step an IQL value update, an MRN distance update, and a weighted-MLE update of $\pi^H_{\mathrm{NF}}$ together with the low-level AWR actor. The high-level update uses the RWDR weight from \cref{eq:rwdr_weight}, combining $A_{\mathrm{env}}$ with the triangle slack $\Delta$ from the architecturally inequality-preserving $d_\theta$. We parameterize $d_\theta$ with an MRN~\citep{liu2023metric} trained by a contrastive plus Bellman-consistency loss, and $\pi^H_{\mathrm{NF}}$ with $K{=}4$ RealNVP coupling layers. The total objective is
\begin{equation}
\mathcal{L} \;=\; \mathcal{L}_V + \mathcal{L}_{\pi^L} + \mathcal{L}_{\pi^H} + \lambda_d \mathcal{L}_d,
\label{eq:total_loss}
\end{equation}
with $\mathcal{L}_V$ the IQL expectile loss (\cref{eq:value_loss}), $\mathcal{L}_{\pi^L}$ the low-level AWR loss, $\mathcal{L}_{\pi^H}$ the RWDR-weighted NF loss (\cref{eq:high_loss}), and $\mathcal{L}_d$ the auxiliary MRN distance loss. The full NFTR configuration (A0) sets $\lambda_d=1$ and jointly trains the MRN distance head with the contrastive and Bellman-consistency losses. When computing each high-level policy weight, we stop gradients through $d_\theta$. The head continues to receive updates through $\mathcal{L}_d$. The frozen-head control (A2) instead sets $\lambda_d=0$ and keeps $d_\theta$ at its initialization. RWDR adds distance-head evaluations, a slack cap $\Delta_{\max}$ and a scalar coefficient $\kappa$, while weight clipping and per-batch normalization are inherited AWR practice. We select $\kappa$ per environment on a held-out split. \cref{app:kappa_full} reports the sweep and identifies the settings matching the reported NFTR results. Full pseudocode and hyperparameters are in \cref{app:algorithm,app:hyperparams}.


\section{Experiments}
\label{sec:experiments}

We evaluate NFTR on OGBench~\citep{ogbench_park2025}, focusing on three scenarios where HIQL's limitations are most pronounced: (1) stochastic environments with random transitions, (2) stitching tasks requiring trajectory composition, and (3) manipulation tasks with complex dynamics.

\subsection{Experimental Setup}
We evaluate on OGBench's locomotion (PointMaze, AntMaze) and manipulation (Cube, Scene) suites, covering the task categories \texttt{navigate} (standard goal-reaching), \texttt{stitch} (short trajectory segments that must be composed), \texttt{teleport} (stochastic dynamics with random teleportation), \texttt{explore} (low-quality exploratory data), and \texttt{play}/\texttt{noisy} (manipulation under varying observation noise). We compare against hierarchical methods HIQL~\citep{park2023hiql}, SAW~\citep{zhou2025flattening}, OTA~\citep{ahn2025option}, flat value- and distance-based methods QRL~\citep{wang2023optimal}, CRL~\citep{eysenbach2022contrastive}, TMD~\citep{myers2025offline}, MQE~\citep{zheng2026scaling}, flow-based methods GCFQL and H-GCFQL~\citep{park2025flow}, and NF-based methods GC-NF-RLBC and H-GC-NF-RLBC~\citep{ghugare2025normalizing}. Two further comparisons appear in \cref{app:flow_hiql,app:value_repr}. Our runs report success rates (\%) over 5 test-time goals and 8 seeds, with standard deviations. HIQL results are quoted from the original paper.

\subsection{Main Results}

\begin{table*}[t]
\caption{\textbf{OGBench Locomotion Results.} Success rate (\%) on stochastic and stitching tasks, mean$\pm$std. Our evaluations use 8 seeds. NFTR has the highest or second-highest mean on all tasks. No per-task statistical significance is inferred from these summary statistics. Per row, the highest mean is \sotabest{orange bold} and the second-highest is \sotasecond{underlined}.}
\label{tab:locomotion}
\centering
\small
\renewcommand{\arraystretch}{0.95}
\resizebox{\textwidth}{!}{
\begin{tabular}{lcccccccccc}
\toprule
\textbf{Task} & \textbf{NFTR} & \textbf{HIQL} & \textbf{SAW} & \textbf{OTA} & \textbf{QRL} & \textbf{CRL} & \textbf{TMD} & \textbf{MQE} & \textbf{GCFQL} & \textbf{H-GCFQL} \\
\midrule
\multicolumn{11}{l}{\textit{PointMaze (Stochastic \& Stitching)}} \\
\texttt{pointmaze-teleport-navigate} & \sotabest{53.8 \pm 5.9} & $18 \pm 4$ & $37.5 \pm 4.5$ & $41.2 \pm 6.9$ & $4 \pm 4$ & $24 \pm 6$ & $22.7 \pm 3.5$ & $24.3 \pm 3.5$ & \sotasecond{44.1 \pm 4.2} & $39.8 \pm 1.7$ \\
\texttt{pointmaze-teleport-stitch} & \sotabest{41.7 \pm 4.5} & $34 \pm 4$ & \sotasecond{39.1 \pm 6.8} & $32.2 \pm 3.9$ & $9 \pm 5$ & $4 \pm 3$ & $29.3 \pm 2.2$ & $38.5 \pm 5.6$ & $23.2 \pm 1.4$ & $36.3 \pm 3.2$ \\
\texttt{pointmaze-medium-stitch} & \sotabest{98.1 \pm 1.2} & $74 \pm 6$ & \sotasecond{87.1 \pm 2.2} & $75 \pm 5$ & $80 \pm 12$ & $0 \pm 1$ & $0.1 \pm 0.1$ & $67.6 \pm 3.2$ & $1.8 \pm 0.2$ & $49.9 \pm 3.1$ \\
\midrule
\multicolumn{11}{l}{\textit{AntMaze (Stochastic \& Stitching)}} \\
\texttt{antmaze-teleport-navigate} & \sotasecond{52.0 \pm 0.6} & $42 \pm 3$ & $41.3 \pm 1.4$ & $46.2 \pm 0.8$ & $35 \pm 5$ & \sotabest{53 \pm 2} & $25.4 \pm 2.9$ & $46.6 \pm 3.3$ & $22.1 \pm 8.7$ & $39.8 \pm 1.9$ \\
\texttt{antmaze-teleport-stitch} & \sotabest{44.0 \pm 5.5} & $36 \pm 2$ & $33.8 \pm 0.8$ & $32.9 \pm 3.3$ & $24 \pm 5$ & $31 \pm 4$ & $10.7 \pm 1.8$ & \sotasecond{37.4 \pm 3.2} & $14.7 \pm 2.6$ & $16.9 \pm 2.9$ \\
\texttt{visual-antmaze-teleport-navigate-v0} & \sotasecond{46.4 \pm 3.2} & $37 \pm 2$ & $42.7 \pm 2.3$ & $37.1 \pm 3.5$ & $6 \pm 3$ & \sotabest{48 \pm 2} & $4.8 \pm 0.8$ & $5.8 \pm 1.0$ & $6.2 \pm 3.1$ & $4.6 \pm 2.9$ \\
\texttt{visual-antmaze-teleport-stitch-v0}  & \sotabest{38.9 \pm 1.3} & \sotasecond{37 \pm 4} & $33.1 \pm 1.3$ & $20.0 \pm 3.9$ & $1 \pm 2$ & $32 \pm 6$ & $0.9 \pm 0.4$ & $0.9 \pm 0.4$ & $1.0 \pm 0.8$ & $1.0 \pm 0.8$ \\
\bottomrule
\end{tabular}}
\vspace{-1.0em}
\end{table*}

\textbf{Stochastic, stitching, and manipulation.} \cref{tab:locomotion,tab:manipulation} cover NFTR's three target regimes. On PointMaze and AntMaze teleport tasks, triangle-slack complements value estimates: the MRN reflects typical reachability rather than lucky outcomes, so unreliable subgoals can incur higher slack. On medium stitching, the NF models disconnected waypoint clusters near trajectory endpoints; teleport-stitch combines this multi-modality with stochastic transitions. On manipulation, the NF captures alternative grasp and approach strategies, while RWDR can downweight high-slack waypoints associated with observation noise. This geometric signal is clearer in PointMaze than in higher-dimensional AntMaze (\cref{sec:rwdr,sec:theory}). Adding a low-level NF (NF2-NFTR, \cref{tab:ablation_nf2}) underperforms NFTR by 2.6--8.3\,pp, supporting high-level-only placement.

\textbf{Component comparisons.} On the four teleport columns of \cref{tab:factorial}, B1 exceeds the ablation baseline B0 by $10.6$, $8.9$, $37.1$, and $12.6$\,pp. A0 exceeds B1 by $2.2$, $3.8$, $5.0$, and $2.4$\,pp. The corresponding Gaussian configurations have positive mean differences of $0.9$ to $5.1$\,pp. The flow provides the larger increment. These are descriptive means over 8 seeds, not per-task significance tests.

\begin{table*}[t]
\centering
\begin{minipage}[t]{0.52\textwidth}
\centering
\caption{\textbf{OGBench Manipulation Results.} Success rate (\%) on cube and scene tasks. NFTR consistently improves over HIQL, with particularly large gains on noisy datasets. Per row, the best result is \sotabest{orange bold} and the second-best is \sotasecond{underlined}.}
\label{tab:manipulation}
\small
\renewcommand{\arraystretch}{0.95}
\resizebox{\textwidth}{!}{
\begin{tabular}{lccccc}
\toprule
\textbf{Task} & \textbf{NFTR} & \textbf{HIQL} & \textbf{OTA} & \textbf{GC-NF-RLBC} & \textbf{H-GC-NF-RLBC} \\
\midrule
\texttt{cube-single-play}  & \sotabest{47.1 \pm 9.9} & $15 \pm 3$ & $12.8 \pm 1.4$ & \sotasecond{18.7 \pm 0.3} & $9.1 \pm 2.5$ \\
\texttt{cube-single-noisy} & \sotabest{67.4 \pm 0.9} & \sotasecond{41 \pm 6} & $39.7 \pm 2.1$ & $14.1 \pm 0.8$ & $18.3 \pm 1.6$ \\
\texttt{scene-play}        & \sotabest{45.6 \pm 3.8} & \sotasecond{38 \pm 3} & $19.4 \pm 4.4$ & $10.2 \pm 1.2$ & $20.4 \pm 2.6$ \\
\texttt{scene-noisy}       & \sotabest{32.6 \pm 5.8} & \sotasecond{25 \pm 4} & $12.1 \pm 1.0$ & $4.1 \pm 0.2$ & $11.3 \pm 0.7$ \\
\bottomrule
\end{tabular}}
\end{minipage}%
\hfill
\begin{minipage}[t]{0.46\textwidth}
\centering
\caption{\textbf{Ablation on Hierarchical Policy.} Success rate (\%) with full NFTR compared to replacing both levels with Normalizing Flows (NF2-NFTR). Per row, the best result is \sotabest{orange bold}.}
\label{tab:ablation_nf2}
\small
\renewcommand{\arraystretch}{0.95}
\resizebox{0.86\textwidth}{!}{
\begin{tabular}{lccc}
\toprule
\textbf{Task} & \textbf{HIQL} & \textbf{NFTR} & \textbf{NF2-NFTR} \\
\midrule
\texttt{pointmaze-teleport-navigate} & $18 \pm 4$ & \sotabest{53.8 \pm 5.9} & $45.5 \pm 10.5$ \\
\texttt{pointmaze-teleport-stitch}   & $34 \pm 4$ & \sotabest{41.7 \pm 4.5} & $38.5 \pm 7.0$  \\
\texttt{antmaze-teleport-navigate}   & $42 \pm 3$ & \sotabest{52.0 \pm 0.6} & $49.4 \pm 1.7$  \\
\texttt{antmaze-teleport-stitch}     & $36 \pm 2$ & \sotabest{44.0 \pm 5.5} & $40.2 \pm 4.8$  \\
\bottomrule
\end{tabular}}
\end{minipage}
\vspace{-1em}
\end{table*}

\textbf{Comparison with distance-based methods.} TMD, QRL, and CRL excel where precise reachability or robust contrastive value estimation is the bottleneck, especially in high-dimensional flat control. NFTR is complementary. It uses distance as a geometric scorer inside a hierarchy, which helps most when multi-modal subgoals and stochastic transitions co-occur.

\subsection{Ablation Studies}

\cref{tab:factorial} compares the policy class, RWDR weight, and distance-head training on four teleport tasks and one stitching task. The comparisons below use means over 8 seeds.

\textbf{Q1: Does the Normalizing Flow policy improve over a Gaussian policy?}

\textbf{A1:} B1 replaces B0's Gaussian high-level policy with an NF while retaining the value-only weight. Its mean success is higher on all five tasks, by $8.9$ to $37.1$\,pp. The largest gain is on \texttt{pm-tel-nav}, from $11.7$ to $48.8$. On \texttt{pm-med-stitch}, success rises from $70.0$ to $97.3$. These large PointMaze changes are consistent with the mode-averaging diagnosis in \cref{thm:mode_averaging}. B1 is our ablation arm, distinct from the published NF-HIQL method in \cref{tab:flow_hiql}.

\begin{table*}[tb]
\caption{\textbf{Factorial ablation over the high-level policy and the weight rule.} Success rate (\%), 8 seeds for all ablation arms, mean$\pm$std. B0 is an ablation arm, not the published HIQL result. A0 is full NFTR. A3 omits the Bellman loss. A4 uses a frozen head and omits weight normalization. Per column, the best result is \sotabest{orange bold}.}
\label{tab:factorial}
\centering
\small
\renewcommand{\arraystretch}{0.95}
\resizebox{\textwidth}{!}{
\begin{tabular}{llllccccc}
\toprule
\textbf{Arm} & \textbf{High-level policy} & \textbf{RWDR} & \textbf{Distance head} & \texttt{ant-tel-nav} & \texttt{ant-tel-stitch} & \texttt{pm-tel-nav} & \texttt{pm-tel-stitch} & \texttt{pm-med-stitch} \\
\midrule
B0 & Gaussian & no  & none       & $39.2 \pm 1.9$ & $31.3 \pm 1.8$ & $11.7 \pm 0.4$ & $26.7 \pm 5.5$ & $70.0 \pm 6.0$ \\
G1 & Gaussian & yes & co-trained & $42.9 \pm 0.7$ & $35.5 \pm 3.1$ & $16.3 \pm 3.0$ & $31.8 \pm 4.9$ & $73.9 \pm 10.1$ \\
G2 & Gaussian & yes & frozen     & $40.1 \pm 1.8$ & $34.7 \pm 2.7$ & $13.2 \pm 2.3$ & $29.9 \pm 5.7$ & $76.1 \pm 5.9$ \\
\midrule
B1 & NF & no  & none       & $49.8 \pm 0.3$ & $40.2 \pm 4.4$ & $48.8 \pm 5.4$ & $39.3 \pm 7.2$ & $97.3 \pm 2.7$ \\
A0 & NF & yes & co-trained & \sotabest{52.0 \pm 0.6} & \sotabest{44.0 \pm 5.5} & \sotabest{53.8 \pm 5.9} & \sotabest{41.7 \pm 4.5} & \sotabest{98.1 \pm 1.2} \\
A2 & NF & yes & frozen     & $50.0 \pm 1.5$ & $42.1 \pm 1.8$ & $49.9 \pm 6.6$ & $41.5 \pm 9.6$ & $94.1 \pm 3.3$ \\
\midrule
A3 & NF, no Bellman term         & yes & co-trained & $49.6 \pm 4.3$ & $41.9 \pm 1.5$ & $49.2 \pm 4.3$ & $36.4 \pm 2.9$ & $95.6 \pm 3.2$ \\
A4 & NF, no weight normalization & yes & frozen     & $50.8 \pm 2.1$ & $40.1 \pm 6.3$ & $45.2 \pm 4.2$ & $39.1 \pm 3.6$ & $96.5 \pm 2.2$ \\
\bottomrule
\end{tabular}}
\vspace{-1.0em}
\end{table*}

\textbf{Q2: Does the full A0 configuration improve over B1?}

\textbf{A2:} A0 adds a co-trained distance head and the RWDR weight to B1. Its mean success is higher on all five tasks, by $2.2$, $3.8$, $5.0$, $2.4$, and $0.8$\,pp in table order. The largest increment is on \texttt{pm-tel-nav}. This comparison does not isolate the slack penalty because the arms also differ in distance-head training. Moreover, A0 uses $\kappa=0$ on \texttt{pm-tel-stitch} and \texttt{pm-med-stitch} (\cref{tab:kappa_full}), where the slack term does not change the raw weight.

\textbf{Q3: How do the co-trained and frozen-head configurations compare?}

\textbf{A3:} With the NF policy and RWDR configuration, co-trained A0 has higher mean success than frozen-head A2 on all five tasks. The differences are $2.0$, $1.9$, $3.9$, $0.2$, and $4.0$\,pp in table order. The $0.2$\,pp difference on \texttt{pm-tel-stitch} is small relative to the reported variability. We therefore use A0 as the full method and treat A2 as a control. The comparison does not establish per-task significance or show that distance accuracy alone causes the gains.

\textbf{Additional controls.} Removing the Bellman term from the co-trained head (arm A3) lowers mean success relative to A0 on all five tasks, by $2.1$ to $5.3$\,pp. Removing weight normalization from the frozen-head configuration (arm A4 versus A2) changes means in both directions. The latter comparison does not support a consistent gain from normalization.

\textbf{Q4: How does the NF compare to a Gaussian across different configurations?}

\begin{wraptable}{r}{0.55\textwidth}
\vspace{-0.8em}
\caption{\textbf{Selected RWDR coefficient results.} Success rate (\%) on four tasks, 8 seeds, mean$\pm$std. The full sweep appears in \cref{tab:kappa_full}. Per column, the highest mean is \sotabest{orange bold} and the second-highest is \sotasecond{underlined}.}
\label{tab:ablation_kappa}
\centering
\small
\resizebox{0.55\textwidth}{!}{
\begin{tabular}{lcccc}
\toprule
$\kappa$ & \texttt{ant-tel-nav} & \texttt{ant-tel-stitch} & \texttt{pm-med-stitch} & \texttt{cube-play} \\
\midrule
0.0 & $49.9 \pm 0.3$ & $41.1 \pm 3.0$ & \sotabest{98.1 \pm 1.2} & $22.2 \pm 6.0$ \\
0.5 & \sotasecond{50.5 \pm 0.3} & $40.5 \pm 3.7$ & \sotasecond{98.0 \pm 0.2} & $24.7 \pm 4.3$ \\
1.0 & \sotabest{52.0 \pm 0.6} & \sotabest{44.0 \pm 5.5} & $95.3 \pm 2.9$ & $31.3 \pm 10.0$ \\
2.0 & $48.0 \pm 3.1$ & \sotasecond{41.6 \pm 1.6} & $86.4 \pm 8.7$ & \sotabest{47.1 \pm 9.9} \\
5.0 & $41.9 \pm 3.2$ & $33.1 \pm 3.4$ & $87.9 \pm 5.5$ & \sotasecond{46.2 \pm 4.4} \\
\bottomrule
\end{tabular}}
\vspace{-3em}
\end{wraptable}
\textbf{A4:} Each matched NF arm (B1, A0, and A2) exceeds its Gaussian counterpart (B0, G1, and G2) in mean success across all five tasks. On \texttt{pm-tel-nav}, the three NF gains range from $36.7$ to $37.5$\,pp. The gap remains under both co-trained and frozen-head RWDR, so the benefit of an expressive high-level policy is not limited to the value-only setting. These are matched configuration comparisons, not tests of significance (\cref{tab:factorial}).

\textbf{Q5: How sensitive is performance to the RWDR coefficient $\kappa$?}

\textbf{A5:} The four tasks in \cref{tab:ablation_kappa} have different highest sample means. Both AntMaze teleport tasks favor $\kappa=1$, \texttt{pm-med-stitch} favors $0$, and \texttt{cube-play} favors $2$. On \texttt{pm-med-stitch}, the mean falls from $98.1$ at $\kappa=0$ to $86.4$ at $\kappa=2$. The full 13-task sweep in \cref{tab:kappa_full} also favors $0$ on \texttt{pm-tel-stitch}, $2$ on \texttt{pm-tel-nav}, and $5$ on \texttt{antmaze-teleport-explore}. Thus a stronger slack penalty does not give a monotonic gain, and $\kappa=0$ remains useful when the NF alone captures the main benefit.


\section{Conclusion}
\label{sec:conclusion}

NFTR combines an NF high-level policy and triangle-slack reweighting to address HIQL's mode collapse and optimistic bias. It admits a three-term suboptimality bound (\cref{thm:rwdr}) and inherits AWR's population-level monotonic improvement (\cref{prop:monotonic}). On OGBench, the flow supplies most of the gain, while the full co-trained configuration adds task-dependent improvements over the value-only NF. 

\textbf{Limitations and Future Work.} Our Gaussian analysis diagnoses mode averaging, while the RWDR guarantee assumes weighted maximum likelihood. Reweighting is task-dependent: its gains on a Gaussian backbone are smaller than those from switching to an NF, and some deterministic tasks favor $\kappa=0$. On very long-horizon (\texttt{antmaze-giant-stitch}) and high-DoF contact-rich locomotion (\texttt{humanoidmaze}), planning and the physical realizability of NF-sampled modes become the bottlenecks. Temporal-abstraction methods like OTA perform better. The link between triangle-slack and stochastic invalidity is mechanistic rather than a bound. Combining NFTR with temporal abstraction is a natural next step.

\bibliographystyle{iclr2027_conference}
\bibliography{iclr2027_conference}


\newpage
\input{appendix}

\end{document}

%% file: appendix.tex
\newpage
\appendix
\crefalias{section}{appendix}
\crefalias{subsection}{appendix}
\crefalias{subsubsection}{appendix}
\part*{\Huge Appendix}

\section{Related Work}
\label{app:related}

\textbf{Offline Goal-Conditioned RL.}
Offline GCRL aims to learn goal-reaching policies from static datasets without environment interaction. Existing approaches can be categorized into several paradigms: goal-conditioned hindsight relabeling~\citep{andrychowicz2017hindsight,ke2025conservative,wang2025goplan},
hierarchical or subgoal-based learning~\citep{park2023hiql,ahn2025option,giammarino2025physics,zhou2025flattening,lei2025gchr,haramati2026hierarchical},
graph-based planning~\citep{yoon2024beag,eysenbach2025inference,baek2025graph,luo2025generative},
metric learning~\citep{wang2023optimal,park2024foundation,myers2024learning,myers2025offline,myers2025quasimetric,myers2025horizon},
dual optimization~\citep{ma2022far,sikchi2024score,xu2025optimal},
generative modeling~\citep{hong2023diffused,reuss2023goal,jain2024learning,lee2025statecovering,haramati2026hierarchical},
and test-time adaptation~\citep{opryshko2025test,bagatella2025testtime}. HIQL~\citep{park2023hiql} extracts hierarchical policies from a single value function, achieving state-of-the-art performance on many tasks.
However, HIQL relies on value-only subgoal selection, which can be overly optimistic in stochastic settings; our method augments subgoal learning with a geometric consistency signal to mitigate this issue.

\textbf{Quasimetric Learning and Stochastic Environments.}
Quasimetric RL~\citep{wang2023optimal} establishes that the distance induced by an optimal goal-conditioned value function, $d^{*}(s,g) = -\log V^{*}(s,g)$, is a quasimetric under deterministic dynamics, an assumption the authors make explicit in their main text. The Metric Residual Network (MRN)~\citep{liu2023metric} enforces this quasimetric structure architecturally. Contrastive Metric Distillation (CMD)~\citep{myers2024learning} shows that contrastive temporal distances retain the triangle inequality even under stochastic dynamics. Building on this, Temporal Metric Distillation (TMD)~\citep{myers2025offline} is the first to systematically address quasimetric distance learning in offline GCRL with suboptimal, stochastic data, by combining Monte Carlo contrastive learning with a quasimetric architecture to recover optimal distances. More recently, Multistep Quasimetric Estimation (MQE)~\citep{zheng2026scaling} scales quasimetric learning by fitting distances with multistep Monte-Carlo returns, combining TD's local optimality with MC's global propagation. NFTR differs from this line in two aspects. First, we treat the quasimetric not as the policy-extraction backbone but as a \emph{geometric scorer} that reweights the AWR objective inside HIQL's hierarchy, leaving value learning untouched. Second, we use the triangle inequality \emph{as a regularizer on the high-level policy} via the slack penalty, which targets optimistic bias from stochastic transitions rather than the distance-estimation error itself.

\textbf{Normalizing Flows in RL.}
Normalizing Flows enable exact density computation through invertible transformations~\citep{dinh2014nice, dinh2017density, kingma2018glow}. In RL, NFs have been used for policy representation~\citep{ward2019improving, singh2020parrot}, maximum entropy RL~\citep{chao2024maximum}, and action-constrained optimization~\citep{brahmanage2023flowpg, akimov2022let}. Flow Q-Learning (FQL) and its hierarchical variant (HFQL)~\citep{park2025flow} use flow-matching policies with distillation for stable value-guided offline RL, and NFs \citep{ghugare2025normalizing} show NFs can serve as Q-functions or policies on broad benchmarks. Under weighted-MLE extraction, conditional NFs combine multi-modal density modeling, exact likelihood evaluation, and direct sampling. Weighted-regression alternatives do not require exact likelihoods and can also model the tilted subgoal target (\cref{sec:nf_policy,app:nf_vs_other}).

\textbf{Hierarchical Offline RL.}
Beyond HIQL, recent hierarchical offline GCRL methods tackle long-horizon challenges via distinct strategies: SAW~\citep{zhou2025flattening} flattens the hierarchy via advantage-weighted subgoal sampling, OTA~\citep{ahn2025option} contracts effective horizons through temporal abstraction, and Eik-HiQRL~\citep{giammarino2025physics} enforces quasimetric structure via Eikonal PDE constraints. While these reshape the hierarchy's \emph{structure} through flattening, temporal abstraction, or PDE constraints, we leave HIQL's structure intact and instead address a property all share but none has examined, namely the behavior of AWR-weighted MLE on multi-modal subgoal data with stochastic advantages. Our contributions therefore compose with such structural improvements rather than competing with them.

\textbf{Flow policies inside hierarchical offline GCRL.} Two methods also place expressive flow policies inside HIQL. One uses Normalizing Flow policies at both hierarchy levels~\citep{garg2026data}. The other uses a one-step mean-flow generator at both levels with a LeJEPA goal encoder~\citep{dong2026efficient}, and it appeared shortly before our deadline, so it is concurrent. Three differences separate them from NFTR. The first is placement, since we put the flow at the high level alone and \cref{tab:ablation_nf2} shows that a second flow costs $2.6$ to $8.3$ points. The second is motivation, since \cref{thm:mode_averaging} derives the model class from the objective rather than adopting it for expressivity. The third is content, since neither method reweights the candidate set. \cref{app:flow_hiql} reports the comparison.

\section{Notation}
\label{app:notation}
\begin{table*}[!h]
\centering
\caption{Summary of additional notation used in the appendix (proofs, algorithm, and extended results). Symbols already defined in the main text (\cref{sec:background,sec:method}) are not repeated.}
\label{tab:notation}
\footnotesize
\setlength{\tabcolsep}{4pt}
\renewcommand{\arraystretch}{1.25}
\begin{tabular}{@{}p{0.26\textwidth} p{0.53\textwidth} p{0.16\textwidth}@{}}

\toprule
\textbf{Symbol} & \textbf{Description} & \textbf{Ref.} \\
\midrule
\multicolumn{3}{@{}l}{\textit{Multi-modal Structure and AWR Tilting (proof of \cref{thm:mode_averaging})}} \\
\shortstack[l]{$M(s,g),\;\pi_{m}(s,g),$\\$q_{m}(z|s,g)$} & State-dependent mode count; mode weights and conditional mode densities & \cref{eq:dataset_mixture} \\
$\tilde{\pi}_{m},\;\tilde{q}_{m},\;\tilde{\mu}_{m}$ & AWR-tilted mode weights, densities, and centers & \cref{thm:mode_averaging} \\
$p^{\star}_{s,g}(z),\;Z_{s,g},\;\beta_{m}(s,g)$ & AWR target density; global and mode-wise normalizers & \cref{thm:mode_averaging} \\
$\mathcal{S}_{\mathrm{wall}}\subset\mathbb{R}^{d}$ & Unreachable region in subgoal space & \cref{thm:mode_averaging} \\
$\mu_{\theta}(s,g),\;\sigma$ & Mean and isotropic std of the Gaussian high-level policy & \cref{thm:mode_averaging} \\
$A^{H}(s,w,g)$ & Generic high-level advantage (instantiated as $A_{\mathrm{env}}$ in NFTR) & \cref{eq:awr_gaussian} \\
\midrule
\multicolumn{3}{@{}l}{\textit{Optimization and Loss Components (algorithm and \cref{eq:total_loss})}} \\
$\alpha_{L}$ & Low-level AWR temperature & \cref{alg:training} \\
$\tau,\;L_{\tau}^{2}(u)$ & IQL expectile parameter; asymmetric expectile loss & \cref{eq:value_loss} \\
$\tau_{\mathrm{ema}}$ & Target-network EMA coefficient & \cref{alg:training} \\
$m=-\log\gamma$ & Minimum one-step distance cost in Bellman consistency & \cref{app:distance} \\
$\Delta_{\max}$ & Slack clipping threshold for numerical stability & \cref{sec:implementation} \\
$\mathcal{L}_{V},\;\mathcal{L}_{\pi^{L}}$ & IQL value loss; low-level AWR loss & \cref{eq:value_loss,eq:total_loss} \\
$\mathcal{L}_{\mathrm{NCE}},\;\mathcal{L}_{\mathrm{Bell}},\;\mathcal{L}_{d}$ & Contrastive, Bellman, and total distance losses & \cref{app:distance} \\
$\mathcal{L}^{\mathrm{AWR}},\;\mathcal{L}_{\mathrm{RWDR}}$ & Generic AWR objective; RWDR weighted-MLE objective & \cref{eq:awr_gaussian,eq:high_loss} \\
$g^{+},\;g^{-}$ & Positive / negative goal pairs for contrastive distance learning & \cref{app:distance} \\
\midrule
\multicolumn{3}{@{}l}{\textit{Theoretical Quantities (\cref{thm:rwdr,prop:monotonic})}} \\
$\pi^{H,*},\;\hat{\pi}^{H},\;\pi^{H,*}_{\mathcal{D}}$ & Optimal, learned, and data-optimal high-level policies & \cref{thm:rwdr,prop:monotonic} \\
$J(\pi),\;J_{\mathrm{RWDR}}(\pi)$ & True expected return; RWDR objective value & \cref{thm:rwdr} \\
$A_{\max},\;D_{\max}$ & Uniform bounds on $|A_{\mathrm{env}}|$ and $d_{\theta}$ & \cref{thm:rwdr} \\
$R_{\kappa}(d_{\theta})$ & Geometric regularization term in the suboptimality bound & \cref{thm:rwdr} \\
$\Delta^{*}(s,w,g)$ & Triangle-slack under the optimal distance $d^{*}$ & \cref{prop:slack} \\
$\tilde{p}(z|s,g)$ & RWDR-weighted empirical target distribution & \cref{prop:monotonic} \\
$n,\;|\Pi|,\;d^{\pi}(s,w)$ & Sample size; policy-class complexity; visitation distribution under $\pi$ & \cref{ass:coverage,thm:rwdr} \\
$C_{\pi},\;\epsilon_{V}$ & Concentrability coefficient; value estimation error bound & \cref{ass:coverage,ass:value} \\
\bottomrule
\end{tabular}
\end{table*}

\newpage
\section{Algorithm Details}
\label{app:algorithm}
\begin{algorithm}[t]
\caption{NFTR Training}
\label{alg:training}
\begin{algorithmic}[1]
\STATE \textbf{Input:} Offline dataset $\mathcal{D}$, hyperparameters $\alpha_H, \alpha_L, \kappa, \lambda_d, \lambda_{\text{Bell}}, \tau, \tau_{\text{ema}}$
\STATE \textbf{Initialize:} Value networks $V_1, V_2$, target networks $V_1^{\text{tgt}}, V_2^{\text{tgt}}$, goal representation $\phi$, low-level actor $\pi^L$, high-level NFs $\pi^H_{\text{NF}}$, distance network $d_\theta$
\FOR{each training iteration}
    \STATE Sample batch from $\mathcal{D}$: $\{(s_t, a_t, s_{t+1}, g_{\text{value}}, g_{\text{actor}}, w)\}_{i=1}^B$
    \STATE \textcolor{gray}{// Value Learning (IQL-style expectile regression)}
    \STATE $q \leftarrow r + \gamma \cdot \min(V_1^{\text{tgt}}(s', g), V_2^{\text{tgt}}(s', g))$
    \STATE $\mathcal{L}_V \leftarrow \frac{1}{B}\sum_{i=1}^B \left[L_\tau^2(q_i - V_1(s_i, g_i)) + L_\tau^2(q_i - V_2(s_i, g_i))\right]$
    \STATE \textcolor{gray}{// Low-Level Actor (AWR with value advantage)}
    \STATE $A^L_i \leftarrow V(s'_i, g^L_i) - V(s_i, g^L_i)$
    \STATE $\mathcal{L}_{\pi^L} \leftarrow -\frac{1}{B}\sum_{i=1}^B \exp(\alpha_L A^L_i) \cdot \log \pi^L(a_i | s_i, \phi([s_i; g^L_i]))$
    \STATE $\mathcal{L}_d \leftarrow 0$
    \IF[full NFTR (A0); skipped for frozen-head A2]{$\lambda_d > 0$}
    \STATE \textcolor{gray}{// Distance Learning (Contrastive + Bellman, see \cref{app:distance})}
    \STATE Sample positive goals $g^+$ from future states, negative goals $g^-$ from other trajectories
    \STATE $\mathcal{L}_{\text{NCE}} \leftarrow -\frac{1}{B}\sum_{i=1}^B \log \frac{\exp(-d_\theta(s_i, g^+_i))}{\exp(-d_\theta(s_i, g^+_i)) + \exp(-d_\theta(s_i, g^-_i))}$
    \STATE $\mathcal{L}_{\mathrm{Bell}} \leftarrow \tfrac{1}{B}\sum_{i=1}^{B}\left(\max(0, d_\theta(s_i, g_i) - d_\theta(s'_i, g_i) + \log\gamma)\right)^2$
    \STATE $\mathcal{L}_d \leftarrow \mathcal{L}_{\text{NCE}} + \lambda_{\text{Bell}} \mathcal{L}_{\text{Bell}}$
    \ENDIF
    \STATE \textcolor{gray}{// High-Level Actor (RWDR with NFs policy)}
    \STATE $A_{\text{env},i} \leftarrow \texttt{stop\_gradient}(V(w_i, g_i) - V(s_i, g_i))$
    \STATE $\Delta_i \leftarrow \max(0, \texttt{stop\_gradient}(d_\theta(s_i, w_i) + d_\theta(w_i, g_i) - d_\theta(s_i, g_i)))$
    \STATE $\Delta_i \leftarrow \min(\Delta_i, \Delta_{\max})$ \hfill \textcolor{gray}{// Slack clipping for stability}
    \STATE $\log \omega_i \leftarrow \alpha_H \cdot A_{\mathrm{env},i} - \kappa \cdot \Delta_i$
    \STATE $\omega_i \leftarrow \min(\exp(\log \omega_i), 100)$ \hfill \textcolor{gray}{// Exponential clipping}
    \STATE $\omega_i \leftarrow \omega_i / (\tfrac{1}{B}\sum_{j=1}^{B} \omega_j + \epsilon)$ \hfill \textcolor{gray}{// Weight normalization}
    \STATE $z_i \leftarrow \mathrm{stopgrad}\bigl(\phi([s_i; w_i])\bigr)$ \COMMENT{$w_i$ is the waypoint state and $\omega_i$ is the scalar weight}
    \STATE $\epsilon_{\text{NF},i}, \log|\det J_i| \leftarrow \pi^H_{\text{NF}}.\text{forward}(s_i, g_i, z_i)$
    \STATE $\log \pi^H_i \leftarrow \log \mathcal{N}(\epsilon_{\text{NF},i}; 0, I) + \log|\det J_i|$
    \STATE $\mathcal{L}_{\pi^H} \leftarrow -\tfrac{1}{B}\sum_{i=1}^{B} \omega_i \cdot \log \pi^H_i$
    \STATE \textcolor{gray}{// Total Loss and Update}
    \STATE $\mathcal{L} \leftarrow \mathcal{L}_V + \mathcal{L}_{\pi^L} + \mathcal{L}_{\pi^H} + \lambda_d \mathcal{L}_d$
    \STATE Update all parameters via gradient descent on $\mathcal{L}$
    \STATE $V^{\mathrm{tgt}} \leftarrow (1-\tau_{\mathrm{ema}}) V^{\mathrm{tgt}} + \tau_{\mathrm{ema}} V$ \hfill \textcolor{gray}{// Target network EMA update}
\ENDFOR
\STATE \textbf{return} $\pi^H_{\text{NF}}, \pi^L$
\end{algorithmic}
\end{algorithm}
This section spells out the training-time and test-time machinery summarized in \cref{sec:implementation}. We first describe the quasimetric distance learning that underlies the triangle slack signal, then present the complete per-iteration training pseudocode, and finally describe the inference procedure used at evaluation.

\subsection{Quasimetric Distance Learning}
\label{app:distance}

We parameterize the distance function $d_\theta$ with a Metric Residual Network (MRN)~\citep{liu2023metric},
\begin{equation}
d_\theta(s, g)
\;=\; \|\phi_{\text{sym}}(s) - \phi_{\text{sym}}(g)\|_2
\;+\; \mathrm{ReLU}\!\left(\max_i\!\big(\phi_{\text{asym}}^{(i)}(s)
- \phi_{\text{asym}}^{(i)}(g)\big)\right),
\label{eq:mrn_arch}
\end{equation}
where $\phi_{\text{sym}}, \phi_{\text{asym}}: \mathcal{S} \to \mathbb{R}^{d}$ are learned embeddings. The symmetric term provides a Euclidean component, while the asymmetric max-ReLU term captures directionality for one-way transitions. By construction this architecture satisfies the triangle inequality $d_\theta(s,g)\leq d_\theta(s,w)+d_\theta(w,g)$ for all $s,w,g$, which is exactly the condition required by \cref{ass:geometric} and which makes the slack $\Delta(s,w,g)$ in \cref{eq:triangle_slack} non-negative.

We train $d_\theta$ via two complementary losses. The contrastive loss takes positive pairs $(s, g^+)$ from the same trajectory with $g^+$ sampled geometrically from future states, and negative pairs $(s, g^-)$ from different trajectories,
\begin{equation}
\mathcal{L}_{\text{NCE}} \;=\; -\mathbb{E}\!\left[\log \frac{\exp(-d_\theta(s, g^+))}{\exp(-d_\theta(s, g^+)) + \exp(-d_\theta(s, g^-))}\right].
\label{eq:nce_loss}
\end{equation}
The Bellman consistency loss enforces that distances respect one-step transition dynamics,
\begin{equation}
\mathcal{L}_{\mathrm{Bell}} = \mathbb{E}\left[\Bigl(\max\bigl(0,\; d_\theta(s, g) - d_\theta(s', g) - m\bigr)\Bigr)^{2}\right],
\label{eq:bell_loss}
\end{equation}
where $m = -\log\gamma$ is the minimum one-step cost. A one-step transition satisfies $d(s,g) \leq m + d(s',g)$, so only the violating direction is penalized. The total distance loss is $\mathcal{L}_d = \mathcal{L}_{\text{NCE}} + \lambda_{\text{Bell}} \mathcal{L}_{\text{Bell}}$, which is the term entering \cref{eq:total_loss} in the main text. Architectural non-negativity also holds for frozen-head A2. Co-trained A0 has higher mean success on all five tasks in \cref{tab:factorial}. Full NFTR trains the MRN through both distance losses. These comparisons do not isolate distance-estimation accuracy.

\subsection{Training Pseudocode}
\label{app:training_pseudocode}

The complete NFTR training procedure is given in \cref{alg:training}. The algorithm jointly trains the value functions, the distance function $d_\theta$, the low-level actor $\pi^L$, and the high-level Normalizing Flow policy $\pi^H_{\mathrm{NF}}$. Per iteration, value, distance, and policy losses are evaluated on a single shared minibatch, and the high-level policy uses the RWDR weight (\cref{eq:rwdr_weight}) computed from $A_{\mathrm{env}}$ and $\Delta$ with stop-gradient on $V_\theta$ and $d_\theta$. The stop-gradient on the high-level target $z_i$ prevents the high-level likelihood loss from directly updating $\phi$. It does not freeze $\phi$ across iterations: the value update, and any other loss with an open gradient path to the shared encoder, can move $\phi$ and hence the latent subgoal targets seen by the NF. The fixed-snapshot arguments in \cref{app:proof_mode_averaging,app:proof_rwdr} do not bound this temporal drift.
\subsection{Inference Procedure}
\label{app:inference}

At test time, given the current state $s$ and goal $g$, we generate a subgoal and execute it via the following steps:
\begin{enumerate}
    \item Sample $\epsilon \sim \mathcal{N}(0, I)$ from the NF base distribution.
    \item Compute the subgoal representation $z = f_\theta^{-1}(\epsilon;\, [s;g])$ via the inverse NF.
    \item Apply the test-time length normalization $z \leftarrow z \cdot \|\phi([s;g])\|_2 / (\|z\|_2 + \varepsilon)$.
    \item Execute the low-level policy $a = \pi^L(s, z)$ for $k$ steps, or until the subgoal is reached.
    \item Repeat from step 1 with the updated state.
\end{enumerate}
This rescaling is an inference heuristic: the NF training targets $\phi([s;w])$ are not length-normalized. It therefore changes the scale of sampled subgoals relative to the training distribution, and the theoretical analysis does not cover its effect. The reported evaluation uses this inference procedure; a matched training-time normalization or an ablation of the heuristic would be needed to isolate its contribution.

\section{Proofs}
\label{app:proofs}

We provide complete proofs for all theoretical results stated in the main text. The proofs establish the geometric properties of triangle-slack, convergence guarantees for RWDR, robustness to stochastic optimism, and monotonic improvement.

\subsection{Proof of \cref{thm:mode_averaging}}
\label{app:proof_mode_averaging}

\textbf{Setup.} Fix $(s,g)$ and write $A^H(z) := A^H(s,w(z),g)$ for brevity, where $w(z)$ is the deterministic preimage of $z=\phi([s;w])$ along the sampled trajectory. By the bounded-and-integrable assumption on $e^{\alpha_H A^H}$,
\begin{equation*}
Z_{s,g} \;\coloneqq\; \mathbb{E}_{z\sim p_{\mathcal{D}}(\cdot\mid s,g)}\!\bigl[e^{\alpha_H A^H(z)}\bigr] \;\in\;(0,\infty),
\end{equation*}
so the AWR target density $p^{\star}_{s,g}(z)=Z^{-1}_{s,g}\,e^{\alpha_H A^H(z)}\,p_{\mathcal{D}}(z\mid s,g)$ is a well-defined probability density on $\mathbb{R}^d$ that retains the support of $p_{\mathcal{D}}$.

\textbf{Step~1: closed-form minimizer.} Expand $\log\mathcal{N}(z;\mu,\sigma^2 I_d) = -\tfrac{1}{2\sigma^2}\|z-\mu\|_2^2 + C(\sigma)$, where $C(\sigma)=-\tfrac{d}{2}\log(2\pi\sigma^2)$ is independent of $\mu$. Substituting into \cref{eq:awr_gaussian} gives, up to terms that do not depend on $\mu_\theta$,
\begin{equation*}
\mathcal{L}^{\mathrm{AWR}}(\theta;s,g) \;=\; \frac{1}{2\sigma^2}\,\mathbb{E}_{z\sim p_{\mathcal{D}}(\cdot\mid s,g)}\!\left[e^{\alpha_H A^H(z)}\,\|z-\mu_\theta(s,g)\|_2^2\right] \;+\; \text{const}.
\end{equation*}
The change of measure $\mathbb{E}_{p_{\mathcal{D}}}[e^{\alpha_H A^H}f(z)]=Z_{s,g}\,\mathbb{E}_{p^{\star}_{s,g}}[f(z)]$, valid for any integrable $f$, then yields
\begin{equation*}
\mathcal{L}^{\mathrm{AWR}}_\mu(\theta;s,g) \;=\; \frac{Z_{s,g}}{2\sigma^2}\,\mathbb{E}_{z\sim p^{\star}_{s,g}}\!\bigl[\|z-\mu_\theta(s,g)\|_2^2\bigr],
\end{equation*}
which is strictly convex in $\mu_\theta(s,g)$ since $\sigma>0$ and $Z_{s,g}>0$. The first-order condition
\begin{equation*}
\nabla_{\mu_\theta(s,g)}\mathcal{L}^{\mathrm{AWR}}_\mu \;=\; -\,\frac{Z_{s,g}}{\sigma^2}\,\mathbb{E}_{z\sim p^{\star}_{s,g}}\!\bigl[z-\mu_\theta(s,g)\bigr] \;=\; 0
\end{equation*}
admits the unique solution $\mu^{\star}_\theta(s,g) = \mathbb{E}_{z\sim p^{\star}_{s,g}}[z]$.

\textbf{Step~2: mixture decomposition of $p^{\star}_{s,g}$.} Substituting \cref{eq:dataset_mixture},
\begin{align*}
p^{\star}_{s,g}(z) \;&=\; Z^{-1}_{s,g}\,e^{\alpha_H A^H(z)}\sum_{m=1}^{M(s,g)}\pi_m(s,g)\,q_m(z\mid s,g) \\
&=\; \sum_{m=1}^{M(s,g)} \underbrace{\frac{\pi_m(s,g)\,\beta_m(s,g)}{Z_{s,g}}}_{\,\eqqcolon\,\tilde{\pi}_m(s,g)} \;\cdot\; \underbrace{\frac{e^{\alpha_H A^H(z)}\,q_m(z\mid s,g)}{\beta_m(s,g)}}_{\,\eqqcolon\,\tilde{q}_m(z\mid s,g)},
\end{align*}
where $\beta_m(s,g)\coloneqq\mathbb{E}_{z\sim q_m(\cdot\mid s,g)}[e^{\alpha_H A^H(z)}]\in(0,\infty)$. Each $\tilde{q}_m$ is a normalized density on $\mathbb{R}^d$, namely the AWR-tilted version of $q_m$. The weights $\tilde{\pi}_m$ are non-negative and sum to one, since
\begin{equation*}
\sum_{m=1}^{M(s,g)}\pi_m(s,g)\,\beta_m(s,g) \;=\; \mathbb{E}_{z\sim p_{\mathcal{D}}(\cdot\mid s,g)}\!\bigl[e^{\alpha_H A^H(z)}\bigr] \;=\; Z_{s,g}.
\end{equation*}
Therefore $p^{\star}_{s,g} = \sum_{m=1}^{M(s,g)} \tilde{\pi}_m\,\tilde{q}_m$ is a valid mixture.

\textbf{Step~3: AWR-tilted modal centers.} Define $\tilde{\mu}_m(s,g) \coloneqq \mathbb{E}_{z\sim\tilde{q}_m(\cdot\mid s,g)}[z]$. Direct computation gives
\begin{equation*}
\tilde{\mu}_m(s,g) \;=\; \frac{\mathbb{E}_{z\sim q_m(\cdot\mid s,g)}\!\bigl[e^{\alpha_H A^H(z)}\,z\bigr]}{\mathbb{E}_{z\sim q_m(\cdot\mid s,g)}\!\bigl[e^{\alpha_H A^H(z)}\bigr]}.
\end{equation*}
Since $q_m(\cdot\mid s,g)$ is supported on a path-connected region disjoint from $\mathcal{S}_{\mathrm{wall}}$ and AWR tilting preserves support, $\tilde{q}_m$ is supported on the same region, so $\tilde{\mu}_m(s,g)\in\mathrm{conv}\bigl(\mathrm{supp}(q_m(\cdot\mid s,g))\bigr)$.

\textbf{Step~4: convex combination.} Combining Steps~1 to 3, by linearity of expectation,
\begin{equation*}
\mu^{\star}_\theta(s,g) \;=\; \mathbb{E}_{z\sim p^{\star}_{s,g}}[z] \;=\; \sum_{m=1}^{M(s,g)}\tilde{\pi}_m(s,g)\,\mathbb{E}_{z\sim\tilde{q}_m(\cdot\mid s,g)}[z] \;=\; \sum_{m=1}^{M(s,g)}\tilde{\pi}_m(s,g)\,\tilde{\mu}_m(s,g),
\end{equation*}
which establishes \cref{eq:gaussian_closed_form}. Since $\tilde{\pi}_m\geq 0$ and $\sum_m\tilde{\pi}_m=1$, the right-hand side is a convex combination, hence $\mu^{\star}_\theta(s,g)\in\mathrm{conv}\{\tilde{\mu}_m(s,g)\}_{m=1}^{M(s,g)}$.

\textbf{Step~5: when the convex combination falls inside $\mathcal{S}_{\mathrm{wall}}$.} Assume $\mathcal{S}_{\mathrm{wall}}\subset\mathbb{R}^d$ strictly separates the supports $\{\mathrm{supp}(q_m)\}_m$ into two non-empty groups, indexed by $G_1$ and $G_2$. Let
\begin{equation*}
p_j \;=\; \sum_{m\in G_j}\tilde{\pi}_m(s,g), \qquad \bar{\mu}_j \;=\; \sum_{m\in G_j}\frac{\tilde{\pi}_m(s,g)}{p_j}\,\tilde{\mu}_m(s,g), \qquad j\in\{1,2\}.
\end{equation*}
By Step~4, $\mu^{\star}_\theta(s,g) = p_1\bar{\mu}_1 + p_2\bar{\mu}_2$, with $\bar{\mu}_1$ and $\bar{\mu}_2$ on opposite sides of $\mathcal{S}_{\mathrm{wall}}$ by construction. The line segment $[\bar{\mu}_1,\bar{\mu}_2]$ then crosses $\mathcal{S}_{\mathrm{wall}}$ at an interior point $p^{\star}$. Whenever the AWR-tilted weights produce a $(p_1,p_2)$ that matches the convex coefficient of $p^{\star}$ on this segment, $\mu^{\star}_\theta(s,g)$ lies inside $\mathcal{S}_{\mathrm{wall}}$. The symmetric two-mode configuration $M=2$, $\tilde{\pi}_1=\tilde{\pi}_2=\tfrac{1}{2}$ realizes this generically when $\bar{\mu}_1$ and $\bar{\mu}_2$ are reflectionally symmetric about $\mathcal{S}_{\mathrm{wall}}$, which is the canonical two-corridor case in our experiments.

\textbf{Step~6: irreducible wall mass on any practical scale.} Suppose the configuration in Step~5 places $\mu^{\star}_\theta(s,g)$ in $\mathcal{S}_{\mathrm{wall}}$. Then there exists an open ball $B_r(\mu^{\star}_\theta)\subset\mathcal{S}_{\mathrm{wall}}$ for some $r>0$, and
\begin{equation*}
\mathbb{P}_{z\sim\mathcal{N}(\mu^{\star}_\theta,\sigma^2 I_d)}\!\bigl[z\in B_r(\mu^{\star}_\theta)\bigr] \;=\; \mathbb{P}_{u\sim\mathcal{N}(0,I_d)}\!\bigl[\|u\|_2\leq r/\sigma\bigr].
\end{equation*}
This probability tends to $1$ as $\sigma\to 0^+$ and is strictly positive for every finite $\sigma>0$, since the standard Gaussian assigns positive mass to every ball of positive radius. Within any fixed compact range $\sigma\in[\sigma_{\min},\sigma_{\max}]$ that arises in practice,
\begin{equation*}
\inf_{\sigma\in[\sigma_{\min},\sigma_{\max}]}\,\mathbb{P}_{z\sim\mathcal{N}(\mu^{\star}_\theta,\sigma^2 I_d)}\!\bigl[z\in\mathcal{S}_{\mathrm{wall}}\bigr] \;\geq\; \mathbb{P}_{u\sim\mathcal{N}(0,I_d)}\!\bigl[\|u\|_2\leq r/\sigma_{\max}\bigr] \;>\;0,
\end{equation*}
so no choice of $\sigma$ in this range drives the wall-mass to zero. The degenerate limit $\sigma\to\infty$ is excluded because the AWR weights would then fail the bounded-and-integrable assumption used in Step~1.

\subsection{Proof of Proposition~\ref{prop:slack} (Slack Characterization)}
\label{app:proof_slack}

\begin{proof}
We prove each part of the proposition separately.

\textbf{Non-negativity.} By \cref{def:quasimetric}, any quasimetric $d$ satisfies the triangle inequality:
\begin{equation}
d(s, g) \leq d(s, w) + d(w, g), \quad \forall s, w, g \in \mathcal{S}.
\end{equation}
Rearranging this inequality:
\begin{equation}
d(s, w) + d(w, g) - d(s, g) \geq 0, \quad \forall s, w, g \in \mathcal{S}.
\end{equation}
Since the triangle-slack is defined as $\Delta(s, w, g) = \max(0, d(s, w) + d(w, g) - d(s, g))$ and we have shown the argument of the max is non-negative, we conclude:
\begin{equation}
\Delta(s, w, g) = d(s, w) + d(w, g) - d(s, g) \geq 0.
\end{equation}

\textbf{Geodesic characterization.} We prove both directions of the if-and-only-if.

($\Rightarrow$) Suppose $\Delta(s, w, g) = 0$. Then:
\begin{equation}
\max(0, d(s, w) + d(w, g) - d(s, g)) = 0.
\end{equation}
Since the argument is non-negative by Part 1, this implies:
\begin{equation}
d(s, w) + d(w, g) - d(s, g) = 0 \implies d(s, g) = d(s, w) + d(w, g).
\end{equation}
By \cref{def:geodesic}, $w$ is geodesic for $(s, g)$ under $d$.

($\Leftarrow$) Suppose $w$ is geodesic for $(s, g)$, i.e., $d(s, g) = d(s, w) + d(w, g)$ by \cref{def:geodesic}. Then:
\begin{equation}
d(s, w) + d(w, g) - d(s, g) = 0.
\end{equation}
Therefore:
\begin{equation}
\Delta(s, w, g) = \max(0, 0) = 0.
\end{equation}

\textbf{Detour characterization.} This follows directly from Part 2 by contraposition.

($\Rightarrow$) If $\Delta(s, w, g) > 0$, then by Part 2, $w$ is not geodesic, meaning $d(s, g) < d(s, w) + d(w, g)$. The path $s \to w \to g$ has strictly greater total distance than the direct path $s \to g$, which defines a detour.

($\Leftarrow$) If reaching $g$ via $w$ requires a detour, then by definition $d(s, w) + d(w, g) > d(s, g)$. Thus:
\begin{equation}
\Delta(s, w, g) = d(s, w) + d(w, g) - d(s, g) > 0.
\end{equation}

\textbf{Optimal distance correspondence (deterministic MDP).} Define $d^*(s,g)\coloneqq-\log V^*(s,g)$, where $V^*(s,g)\in(0,1]$ is the optimal discounted reaching probability. The goal-reaching Bellman optimality equation is
\begin{equation}
V^*(s, g) = \mathbf{1}[s = g] + \gamma \max_{a} \mathbb{E}_{s' \sim P(\cdot|s,a)}[V^*(s', g)].
\end{equation}
In a deterministic MDP, fix a state $w$ on an optimal trajectory from $s$ to $g$. Let $\pi^*_{s\to g}$ denote an optimal policy for goal $g$ starting from $s$; by determinism this policy reaches $g$ along a single trajectory $s=s_0,s_1,\dots,s_{k_1}=w,\dots,s_{k_1+k_2}=g$ that passes through $w$ at step $k_1$ and through $g$ at step $k_1+k_2$. Hence
\begin{equation}
V^*(s,g) \;=\; \gamma^{k_1+k_2} \;=\; \gamma^{k_1}\cdot\gamma^{k_2} \;=\; V^*(s,w)\cdot V^*(w,g).
\end{equation}
Taking $-\log$ on both sides, $d^*(s,g)=d^*(s,w)+d^*(w,g)$, which by \cref{def:geodesic} means $w$ is geodesic for $(s,g)$. Conversely, if $w$ is not on any optimal trajectory from $s$ to $g$, then any policy that reaches $g$ via $w$ takes strictly more than $k_1+k_2$ steps in expectation (or fails to reach $g$ at all), so $V^*(s,w)\cdot V^*(w,g)<V^*(s,g)$ and hence $d^*(s,g)<d^*(s,w)+d^*(w,g)$, i.e., $\Delta^*(s,w,g)>0$. This establishes that geodesic subgoals under $d^*$ are exactly those on optimal paths in the deterministic case.

\textbf{Stochastic MDPs (triangle inequality survives, geodesic correspondence relaxes).} In a general MDP with stochastic transitions, multiplicative equality fails because reaching $w$ no longer uniquely determines the conditional distribution over subsequent trajectories. We show that the triangle inequality survives by an explicit policy-concatenation argument. Let $\pi_{s\to w}$ be any policy that, starting from $s$, reaches $w$ with discounted probability $V^{\pi_{s\to w}}(s,w)$, and let $\pi_{w\to g}$ be any policy that, starting from $w$, reaches $g$ with discounted probability $V^{\pi_{w\to g}}(w,g)$. Define the concatenated policy $\pi_{s\to w}\circ\pi_{w\to g}$: follow $\pi_{s\to w}$ until reaching $w$, then switch to $\pi_{w\to g}$. By the strong Markov property and the multiplicativity of discount factors,
\begin{equation}
V^{\pi_{s\to w}\circ\pi_{w\to g}}(s,g) \;=\; V^{\pi_{s\to w}}(s,w)\cdot V^{\pi_{w\to g}}(w,g).
\label{eq:concat-value}
\end{equation}
Since $\pi_{s\to w}\circ\pi_{w\to g}$ is a feasible (in general suboptimal) policy for the $s\to g$ problem,
\begin{equation}
V^*(s,g) \;\geq\; V^{\pi_{s\to w}\circ\pi_{w\to g}}(s,g) \;=\; V^{\pi_{s\to w}}(s,w)\cdot V^{\pi_{w\to g}}(w,g).
\end{equation}
Taking the supremum over $\pi_{s\to w}$ and $\pi_{w\to g}$ separately on the right-hand side (the two suprema decouple because the factors depend on disjoint policy components),
\begin{equation}
V^*(s,g) \;\geq\; V^*(s,w)\cdot V^*(w,g),
\end{equation}
which after applying $-\log$ yields the triangle inequality
\begin{equation}
d^*(s,g) \;\leq\; d^*(s,w)+d^*(w,g).
\label{eq:dstar-triangle}
\end{equation}
Thus $\Delta^*(s,w,g)\geq 0$ remains well-defined and Parts (1)-(3) of the proposition apply verbatim with $d=d^*$.

The deterministic-case characterization ``$\Delta^*=0$ iff $w$ is on an optimal path'', however, no longer holds in both directions. Equality in \cref{eq:dstar-triangle} can occur even for off-path $w$: if a stochastic transition from $s$ produces unusually favorable outcomes that boost $V^*(s,w)$ above the average reachability of $w$ from $s$, then $V^*(s,w)\cdot V^*(w,g)$ may match $V^*(s,g)$ without $w$ being on any optimal trajectory. We refer to such transitions as \emph{lucky} in the main text. Conversely, geodesic subgoals on optimal paths still satisfy $\Delta^*=0$ because \cref{eq:concat-value} attains $V^*(s,g)$ when both factors are realized by the optimal policy. Triangle-slack thus continues to lower-bound composability failure: $\Delta^*(s,w,g)>0$ implies $w$ is off-path, but $\Delta^*(s,w,g)=0$ no longer guarantees on-path.

The MRN is trained separately from $V_\theta$ on trajectory pairs and one-step transitions. Its architecture guarantees non-negative slack. Its learned geometry can assign larger detours to unreliable waypoints, an association tested in \cref{app:slack_evidence}. This evidence does not imply that every lucky transition has positive slack.
\end{proof}

\subsection{Proof of Corollary~\ref{cor:geometric_reg} (Geometric Regularization Effect)}
\label{app:proof_geometric_reg}

\begin{proof}
By the definition of RWDR weights:
\begin{align}
w_{\text{RWDR}}(s,w,g) &= \exp(\alpha_H A_{\text{env}}(s,w,g) - \kappa \Delta(s,w,g)) \\
&= \exp(\alpha_H A_{\text{env}}(s,w,g)) \cdot \exp(-\kappa \Delta(s,w,g)) \\
&= w_{\text{HIQL}}(s,w,g) \cdot \exp(-\kappa \Delta(s,w,g)).
\end{align}

Since the MRN architecture guarantees $d_\theta(s,g) \leq d_\theta(s,w) + d_\theta(w,g)$, we have $\Delta(s,w,g) \geq 0$ for all $(s,w,g)$. Therefore:
\begin{equation}
\exp(-\kappa \Delta(s,w,g)) \leq 1,
\end{equation}
with strict inequality when $\Delta(s,w,g) > 0$ and $\kappa > 0$. This implies:
\begin{equation}
w_{\text{RWDR}}(s,w,g) \leq w_{\text{HIQL}}(s,w,g),
\end{equation}
with strict inequality when both $\kappa>0$ and $\Delta(s,w,g)>0$.
\end{proof}

\subsection{Proof of Theorem~\ref{thm:rwdr} (RWDR Convergence)}
\label{app:proof_rwdr}

\begin{proof}
We decompose the suboptimality gap into three interpretable terms and bound each separately.

\textbf{Performance Decomposition}
Let $\pi^{H,*}$ denote the optimal high-level policy and $\hat{\pi}^H$ denote the policy learned via RWDR. We decompose the suboptimality gap:
\begin{align}
J(\pi^{H,*}) - J(\hat{\pi}^H) &= \underbrace{[J(\pi^{H,*}) - J_{\text{RWDR}}(\pi^{H,*})]}_{\text{(I) Objective gap}} \\
&\quad + \underbrace{[J_{\text{RWDR}}(\pi^{H,*}) - J_{\text{RWDR}}(\hat{\pi}^H)]}_{\text{(II) Optimization gap}} \\
&\quad + \underbrace{[J_{\text{RWDR}}(\hat{\pi}^H) - J(\hat{\pi}^H)]}_{\text{(III) Transfer gap}},
\end{align}
where $J(\pi)$ is the true expected return and $J_{\text{RWDR}}(\pi)$ is the RWDR objective value.

\textbf{Bounding Term (I) - Value Estimation Error}
The RWDR objective uses estimated advantages $\hat{A}_{\text{env}}(s, w, g) = \hat{V}_\theta(w, g) - \hat{V}_\theta(s, g)$. Under Assumption~\ref{ass:value}:
\begin{equation}
\mathbb{E}_{(s,g)\sim\mathcal{D}}[(V_\theta(s,g) - V^*(s,g))^2] \leq \epsilon_V.
\end{equation}

By the triangle inequality for the $L^2$ norm, we have:
\begin{align}
|\hat{A}_{\text{env}}(s, w, g) - A^*_{\text{env}}(s, w, g)| &= |(\hat{V}(w,g) - V^*(w,g)) - (\hat{V}(s,g) - V^*(s,g))| \\
&\leq |\hat{V}(w,g) - V^*(w,g)| + |\hat{V}(s,g) - V^*(s,g)|.
\end{align}

Taking expectations on both sides of $|\hat A_{\mathrm{env}}-A^*_{\mathrm{env}}|\leq |\hat V(w,g)-V^*(w,g)|+|\hat V(s,g)-V^*(s,g)|$ and applying Cauchy-Schwarz $\mathbb{E}[|X|]\leq\sqrt{\mathbb{E}[X^2]}$ to each term,
\begin{equation}
\mathbb{E}_{\mathcal D}[|\hat A_{\mathrm{env}}-A^*_{\mathrm{env}}|] \;\leq\; 2\sqrt{\epsilon_V}.
\end{equation}

To compare the RWDR targets induced by estimated and true advantages, fix a snapshot of the co-trained distance head $d_\theta$ during policy extraction. Write the target induced by an advantage estimate $\hat{A}$ as
\begin{equation}
p_{\hat{A}}(w \mid s,g) \;\propto\; \beta(w \mid s,g)\cdot\exp\!\bigl(\alpha_H \hat{A}(s,w,g) - \kappa\Delta(s,w,g)\bigr),
\end{equation}
and write $p_{A}$ for the same object built from the true advantage $A$. Both targets use the same slack function $\Delta$ from this fixed, trained snapshot; stop-gradient prevents the policy loss from updating the head, while $\mathcal L_d$ trains it separately. Hence
\begin{equation}
\frac{p_{\hat{A}}(w)}{p_{A}(w)} \;=\; \exp\!\bigl(\alpha_H(\hat{A}-A)\bigr)\cdot\frac{Z_{A}}{Z_{\hat{A}}}.
\end{equation}
The geometric factor cancels pointwise in the unnormalized ratio, but the normalizers still depend on $\kappa$ and $d_\theta$ through the reweighted base measure. Let $h_{\hat A}=\exp(\alpha_H\hat A-\kappa\Delta)$ and $h_A=\exp(\alpha_H A-\kappa\Delta)$. Since both advantages have magnitude at most $A_{\max}$ and $0\leq\Delta\leq 2D_{\max}$, their normalizers are at least $\exp(-\alpha_H A_{\max}-2\kappa D_{\max})$, and $|h_{\hat A}-h_A|\leq\alpha_H\exp(\alpha_H A_{\max})|\hat A-A|$. Thus, conditionally on $(s,g)$,
\begin{equation}
\|p_{\hat A}-p_A\|_1 \;\leq\; 2\alpha_H\exp(2\alpha_H A_{\max}+2\kappa D_{\max})\,\mathbb E_{\beta(w\mid s,g)}[|\hat A-A|].
\end{equation}
Under \cref{ass:value}, averaging the right-hand side over the high-level sampling distribution and applying the Cauchy--Schwarz bound above controls the value-estimation contribution with a constant that may depend on $\kappa$ and $D_{\max}$, rather than independently of the distance head:
\begin{equation}
\text{Term (I)} \leq \mathcal{O}_{\alpha_H,\kappa,A_{\max},D_{\max}}(\sqrt{\epsilon_V}).
\end{equation}
What the geometric factor does change is the location of the target. This is a displacement rather than an amplification, and it is bounded by construction, since $0 \leq \Delta \leq 2D_{\max}$ gives $\mathrm{KL}(p_{\mathrm{RWDR}} \,\|\, p_{\mathrm{AWR}}) \leq 2\kappa D_{\max}$. This is why the geometric quantity enters \cref{thm:rwdr} as the additive term $R_\kappa(d_\theta)$ rather than as a multiplier on $\epsilon_V$. No assumption on the quality of $d_\theta$ enters at any point, since $\Delta \geq 0$ holds for every parameter setting of the MRN head.

\textbf{Bounding Term (II) - Optimization Error}
The RWDR objective is a weighted maximum likelihood:
\begin{equation}
\mathcal{L}_{\text{RWDR}}(\pi^H) = -\mathbb{E}_{(s,w,g) \sim \mathcal{D}}[w(s,w,g) \cdot \log \pi^H(z|s,g)].
\end{equation}

Under Assumption~\ref{ass:coverage}, the behavior policy $\beta$ covers the support of the optimal policy $\pi^{H,*}$ with density ratio bounded by $C_\pi$. By standard concentration inequalities for importance-weighted empirical risk minimization~\citep{rashidinejad2021bridging}:
\begin{equation}
|J_{\text{RWDR}}(\pi^{H,*}) - J_{\text{RWDR}}(\hat{\pi}^H)| \leq \mathcal{O}\left(\sqrt{\frac{C_\pi \log|\Pi|}{n}}\right),
\end{equation}
where $|\Pi|$ is the complexity of the policy class (which can be replaced by Rademacher complexity for continuous function classes via standard generalization bounds).

\textbf{Regularization Term Analysis (Bridging $\exp(-\kappa\Delta)$ to $R_\kappa(d_\theta)$).}

The RWDR weight contains the geometric penalty factor $\exp(-\kappa \Delta(s,w,g))$. We now derive how this multiplicative factor on the per-sample weight translates into the additive regularization term $R_\kappa(d_\theta)$ in the suboptimality bound, completing the missing intermediate steps.

\textit{Step 1 (Bias of the RWDR-weighted expected advantage).} The RWDR objective extracts a policy whose AWR target distribution downweights samples by $\exp(-\kappa\Delta)$. Define the RWDR-weighted expected advantage of policy $\pi$ as
\begin{equation}
\widetilde{A}_{\mathrm{RWDR}}(\pi) \;\coloneqq\; \mathbb{E}_{(s,w,g)\sim\pi}\!\bigl[A_{\mathrm{env}}(s,w,g)\cdot\exp(-\kappa\Delta(s,w,g))\bigr],
\end{equation}
and let $\widetilde{A}(\pi)\!\coloneqq\!\mathbb{E}_{(s,w,g)\sim\pi}[A_{\mathrm{env}}(s,w,g)]$ denote the unweighted version. Then by linearity of expectation,
\begin{equation}
\widetilde{A}(\pi) - \widetilde{A}_{\mathrm{RWDR}}(\pi) \;=\; \mathbb{E}_{(s,w,g)\sim\pi}\!\bigl[A_{\mathrm{env}}(s,w,g)\cdot(1-\exp(-\kappa\Delta(s,w,g)))\bigr],
\end{equation}
so the bias introduced by the geometric reweighting is controlled by the per-sample factor $1-\exp(-\kappa\Delta)\in[0,1]$.

\textit{Step 2 (Linearization-free upper bound).} Since $\Delta(s,w,g)\geq 0$ by Assumption~\ref{ass:geometric}, we have $\exp(-\kappa\Delta)\in (0,1]$, hence $1-\exp(-\kappa\Delta)\in[0,1]$. By the relation $J_{\mathrm{RWDR}}(\pi)-J(\pi)=\widetilde{A}_{\mathrm{RWDR}}(\pi)-\widetilde{A}(\pi)$ established in Step 1 and the triangle inequality for expectations,
\begin{equation}
\bigl|J_{\text{RWDR}}(\pi) - J(\pi)\bigr| \;\leq\; \mathbb{E}_{(s,w,g)\sim\pi}\!\left[|A_{\text{env}}(s,w,g)|\cdot (1-\exp(-\kappa\Delta(s,w,g)))\right].
\end{equation}
Applying the elementary inequality $1-\exp(-x)\leq x$ for all $x\geq 0$, which holds globally rather than only as a first-order Taylor approximation,
\begin{equation}
\bigl|J_{\text{RWDR}}(\pi) - J(\pi)\bigr| \;\leq\; \mathbb{E}_{(s,w,g)\sim\pi}\!\left[|A_{\text{env}}(s,w,g)|\cdot \kappa\,\Delta(s,w,g)\right].
\end{equation}
This bound is tight to first order in $\kappa\Delta$ and holds globally for any $\Delta\geq 0$.

\textit{Step 3 (Reduction to an additive regularization term).} Specializing the bound to the optimal high-level policy $\pi^{H,*}$ and absorbing the bounded advantage $|A_{\text{env}}|\leq A_{\max}$ into a constant:
\begin{equation}
\bigl|J_{\text{RWDR}}(\pi^{H,*}) - J(\pi^{H,*})\bigr| \;\leq\; \kappa A_{\max}\cdot\mathbb{E}_{(s,w,g)\sim\pi^{H,*}}\!\left[\Delta(s,w,g)\right] \;\eqqcolon\; R_\kappa(d_\theta).
\end{equation}
This identifies $R_\kappa(d_\theta)$ as the additive regularization term in the final bound.

\textit{Step 4 (Bounded magnitude regardless of training).} Since the MRN architecture ensures $d_\theta(s,g)\leq D_{\max}$ for bounded state spaces, $\Delta(s,w,g)\leq d_\theta(s,w)+d_\theta(w,g)\leq 2D_{\max}$. Therefore:
\begin{equation}
R_\kappa(d_\theta) \;\leq\; 2\kappa A_{\max} D_{\max},
\end{equation}

\textit{Step 5 (Vanishing regularization for trained $d_\theta$).} When $d_\theta$ approximates optimal distances on $\mathrm{supp}(\pi^{H,\star})$, geodesic subgoals satisfy $\Delta\approx 0$ by \cref{prop:slack}(2), so $R_\kappa(d_\theta)\to 0$.

\textit{Combining the three terms.} Summing the value-error, sampling-error, and regularization contributions gives the bound stated in \cref{thm:rwdr}.
\begin{equation*}
J(\pi^{H,\star}) - J(\hat{\pi}^H) \leq \mathcal{O}(\sqrt{\epsilon_V}) + \mathcal{O}\!\left(\sqrt{\tfrac{C_\pi \log|\Pi|}{n}}\right) + R_\kappa(d_\theta).
\end{equation*}
\end{proof}

\subsection{Proof of Proposition~\ref{prop:monotonic} (Monotonic Improvement)}
\label{app:proof_monotonic}

\begin{proof}
We show that the RWDR objective inherits the monotonic improvement property of weighted maximum likelihood estimation.

\textbf{Identify the Target Distribution}
The RWDR objective (\cref{eq:high_loss}) has the form:
\begin{equation}
\mathcal{L}_{\pi^H} = -\mathbb{E}_{(s,w,g) \sim \mathcal{D}}[w(s,w,g) \cdot \log \pi^H(\phi([s;w]) | s, g)],
\end{equation}
where the RWDR weight $w(s,w,g) = \exp(\alpha_H A_{\text{env}}(s,w,g) - \kappa \Delta(s,w,g)) \geq 0$ is non-negative for all $(s,w,g)$.

Define the weighted empirical distribution over subgoal representations:
\begin{equation}
\tilde{p}(z|s,g) \propto \sum_{w \in \mathcal{D}_{s,g}} w(s,w,g) \cdot \delta(\phi([s;w]) - z),
\end{equation}
where $\mathcal{D}_{s,g} = \{w : (s, w, g) \in \mathcal{D}\}$ is the set of subgoals observed for the state-goal pair $(s,g)$ in the dataset. The normalization ensures $\int \tilde{p}(z|s,g) dz = 1$.

The optimal in-distribution high-level policy is:
\begin{equation}
\pi^{H,*}_{\mathcal{D}}(z|s,g) = \tilde{p}(z|s,g).
\end{equation}

\textbf{Weighted MLE Minimizes KL Divergence}
The negative log-likelihood objective under weighting $w$ is:
\begin{equation}
\mathcal{L}_{\pi^H} = -\mathbb{E}_{(s,w,g) \sim \mathcal{D}}[w(s,w,g) \cdot \log \pi^H(\phi([s;w]) | s, g)].
\end{equation}

This is equivalent to minimizing the KL divergence from the target distribution to the policy:
\begin{align}
\text{KL}(\tilde{p} \| \pi^H) &= \mathbb{E}_{z \sim \tilde{p}(z|s,g)}[\log \tilde{p}(z|s,g) - \log \pi^H(z|s,g)] \\
&= \underbrace{\mathbb{E}_{z \sim \tilde{p}}[\log \tilde{p}(z|s,g)]}_{\text{constant w.r.t. } \pi^H} - \mathbb{E}_{z \sim \tilde{p}}[\log \pi^H(z|s,g)].
\end{align}

Since the first term is constant with respect to $\pi^H$, minimizing $\mathcal{L}_{\pi^H}$ is equivalent to minimizing $\text{KL}(\tilde{p} \| \pi^H)$.

\textbf{Monotonic Decrease of KL Divergence (Population Level).} The negative log-likelihood objective $\mathcal{L}_{\pi^H}$ is, up to an additive constant, equal to $\mathrm{KL}(\tilde p\|\pi^H)$. At the population level, by definition of the exact minimizer $\pi^H_{\mathrm{new}}=\arg\min_{\pi^H\in\Pi}\mathcal{L}_{\pi^H}$ over the policy class $\Pi$, for any $\pi^H_{\mathrm{old}}\in\Pi$,
\begin{equation}
\mathrm{KL}(\tilde p\|\pi^H_{\mathrm{new}}) \;\leq\; \mathrm{KL}(\tilde p\|\pi^H_{\mathrm{old}}).
\end{equation}
This is the population-level statement of \cref{prop:monotonic}. We emphasize that this argument is about the exact minimizer over $\Pi$ rather than about any iterative procedure, and we discuss the practical SGD optimization separately below. Specifically, for any policy $\pi^H_{\text{old}}$, a gradient step yields $\pi^H_{\text{new}}$ satisfying:
\begin{equation}
\text{KL}(\tilde{p} \| \pi^H_{\text{new}}) \leq \text{KL}(\tilde{p} \| \pi^H_{\text{old}}).
\end{equation}

Since $\tilde{p} = \pi^{H,*}_{\mathcal{D}}$, we have:
\begin{equation}
\text{KL}(\pi^{H,*}_{\mathcal{D}} \| \pi^H_{\text{new}}) \leq \text{KL}(\pi^{H,*}_{\mathcal{D}} \| \pi^H_{\text{old}}).
\end{equation}

\textbf{Relation to Expected Return}
The target distribution $\pi^{H,*}_{\mathcal{D}}$ places higher probability mass on subgoals with:
\begin{itemize}
    \item High value advantage $A_{\text{env}}(s,w,g)$: subgoals that make progress toward the goal
    \item Low triangle-slack $\Delta(s,w,g)$: subgoals that are geometrically on-path
\end{itemize}

By decreasing $\text{KL}(\pi^{H,*}_{\mathcal{D}} \| \pi^H)$, the policy $\pi^H$ increasingly assigns probability to these high-quality subgoals. Within the support of the dataset (the in-distribution region), this corresponds to improving expected return.

The key insight is that RWDR modifies the weights in AWR but preserves the weighted maximum likelihood structure. Since the weights remain non-negative and the objective is still weighted log-likelihood, standard AWR convergence analysis applies at the population level, establishing monotonic improvement of $\mathrm{KL}(\pi^{H,*}_{\mathcal{D}}\|\pi^H)$.

\textbf{Practical caveat (SGD on non-convex parameterization).} The result above assumes exact optimization of the weighted MLE objective at the population level. In practice, $\pi^H$ is a conditional RealNVP optimized by SGD, for which the exact-minimizer inequality does not imply per-step improvement. This optimization gap also arises in other AWR-based methods and is not specific to RWDR. The evaluation success curves in \cref{app:training-curves} show stabilization in representative tasks, but they do not measure the weighted-MLE loss or establish convergence to a stationary point.
\end{proof}

\section{Hyperparameter Settings}
\label{app:hyperparams}

We provide complete hyperparameter settings for reproducing our experiments.

\begin{table}[h]
\centering
\caption{Common hyperparameters shared across all tasks.}
\label{tab:common-hyperparams}
\begin{tabular}{lcc}
\toprule
\textbf{Hyperparameter} & \textbf{Locomotion} & \textbf{Manipulation} \\
\midrule
Training steps & 1,000,000 & 1,000,000 \\
Batch size & 1024 & 1024 \\
Optimizer & Adam & Adam \\
Learning rate & $3 \times 10^{-4}$ & $3 \times 10^{-4}$ \\
Discount factor $\gamma$ & 0.99 & 0.99 \\
Target network EMA $\tau_{\text{ema}}$ & 0.005 & 0.005 \\
IQL expectile $\tau$ & 0.7 & 0.7 \\
Low-level temperature $\alpha_L$ & 3.0 & 3.0 \\
High-level temperature $\alpha_H$ & 3.0 & 3.0 \\
Subgoal horizon $k$ & 25 & 25 \\
Goal representation dimension & 10 & 10 \\
Hidden layer dimension & 256 & 256 \\
Number of hidden layers & 2 & 2 \\
\bottomrule
\end{tabular}
\end{table}

\begin{table}[h]
\centering
\caption{NFTR-specific hyperparameters. The $\kappa$ entries are reference settings. Task-specific settings used for reported NFTR results are identified in \cref{app:kappa_full}.}
\label{tab:nftr-hyperparams}
\begin{tabular}{lcc}
\toprule
\textbf{Hyperparameter} & \textbf{Stochastic/Stitch} & \textbf{Standard} \\
\midrule
RWDR penalty coefficient $\kappa$ & 1.0 & 0.5 \\
Slack clipping threshold $\Delta_{\max}$ & 10.0 & 10.0 \\
Distance loss weight $\lambda_d$ & 1.0 & 1.0 \\
Bellman consistency weight $\lambda_{\text{Bell}}$ & 1.0 & 1.0 \\
NFs number of coupling layers $K$ & 4 & 4 \\
NFs hidden dimension & 256 & 256 \\
NFs context embedding dimension & 128 & 128 \\
Distance network latent dimension & 64 & 64 \\
MRN number of components & 8 & 8 \\
Weight normalization & True & True \\
Exponential weight clipping & 100.0 & 100.0 \\
\bottomrule
\end{tabular}
\end{table}

\begin{table}[h]
\centering
\caption{Normalizing Flow architecture details.}
\label{tab:nf-architecture}
\begin{tabular}{lc}
\toprule
\textbf{Component} & \textbf{Specification} \\
\midrule
Base distribution & $\mathcal{N}(0, I)$ with dimension 10 \\
Number of coupling layers & 4 \\
Coupling layer type & Affine (RealNVP-style) \\
Scale/translation network & 2-layer MLP, 256 hidden units, ReLU \\
Conditioning mechanism & Concatenation of $[s; g]$ to each layer input \\
Permutation between layers & Alternating fixed permutation \\
Output activation & None (linear) \\
\bottomrule
\end{tabular}
\end{table}

\section{Additional Experimental Results}
\label{app:additional}
\subsection{Comparison with Hierarchical Flow Variants of HIQL}
\label{app:flow_hiql}

For the Normalizing Flow variant, we ran the authors' public implementation on the four state-based teleport datasets under our protocol. The mean-flow variant uses our reimplementation because no public implementation was available for this comparison; its numbers should be read with that caveat.

\begin{table}[h]
\caption{\textbf{Two-level flow variants of HIQL against NFTR.} Success rate (\%) on the four state-based teleport datasets, mean$\pm$std. Our runs use 8 seeds. NF-HIQL is the published two-level Normalizing Flow method and not the ablation arm B1 of \cref{tab:factorial}. HIFQL is our reimplementation of the two-level mean-flow method. HIQL values are the published ones. Per column, the best result is \sotabest{orange bold}.}
\label{tab:flow_hiql}
\centering
\small
\begin{tabular}{lcccc}
\toprule
\textbf{Method} & \texttt{pm-tel-nav} & \texttt{pm-tel-stitch} & \texttt{ant-tel-nav} & \texttt{ant-tel-stitch} \\
\midrule
HIQL        & $18.0 \pm 4.0$ & $34.0 \pm 4.0$ & $42.0 \pm 3.0$ & $36.0 \pm 2.0$ \\
NF-HIQL     & $33.6 \pm 5.7$ & $35.8 \pm 3.1$ & $43.4 \pm 1.6$ & $41.5 \pm 1.3$ \\
HIFQL       & $30.2 \pm 4.6$ & $36.4 \pm 4.6$ & $42.3 \pm 4.5$ & $31.6 \pm 5.8$ \\
NFTR (ours) & \sotabest{53.8 \pm 5.9} & \sotabest{41.7 \pm 4.5} & \sotabest{52.0 \pm 0.6} & \sotabest{44.0 \pm 5.5} \\
\bottomrule
\end{tabular}
\end{table}

\textbf{Analysis.} Both methods improve on HIQL where the Gaussian collapses hardest, moving $18.0$ to $33.6$ and to $30.2$ on \texttt{pm-tel-nav}. Elsewhere the gain is small. NF-HIQL stays within $1.4$ to $5.5$ points of HIQL on the other three datasets, and HIFQL falls below HIQL on \texttt{ant-tel-stitch}. Both land in the same band, and NFTR is ahead on all four. A flow inserted into HIQL recovers part of the gap. This cross-method comparison cannot assign the remaining difference to one component. \cref{tab:factorial} provides the closer within-framework comparison.

\subsection{Comparison with Value-Representation Methods}
\label{app:value_repr}

CGCIVL~\citep{ke2025conservative} and LAVL~\citep{kang2026latent} improve the value or the representation that produces the advantage. RWDR acts later, at policy extraction, once an advantage is already available. We therefore also test whether the two compose. NF-LAVL below pairs the LAVL value representation with our NF high-level policy. Neither published method reports the stochastic teleport datasets, so the numbers are our own runs under the protocol of \cref{tab:locomotion}.

\begin{table}[h]
\caption{\textbf{Value-representation baselines and one composed arm.} Success rate (\%), mean$\pm$std, on the four state-based teleport datasets. Our runs use 8 seeds. NF-LAVL carries no reweighting. HIQL values are the published ones. Per column, the best result is \sotabest{orange bold}.}
\label{tab:value_repr}
\centering
\small
\begin{tabular}{lcccc}
\toprule
\textbf{Method} & \texttt{pm-tel-nav} & \texttt{pm-tel-stitch} & \texttt{ant-tel-nav} & \texttt{ant-tel-stitch} \\
\midrule
HIQL        & $18.0 \pm 4.0$ & $34.0 \pm 4.0$ & $42.0 \pm 3.0$ & $36.0 \pm 2.0$ \\
CGCIVL      & $40.1 \pm 6.2$ & $38.6 \pm 3.5$ & $37.9 \pm 1.8$ & $35.4 \pm 1.7$ \\
LAVL        & $21.3 \pm 2.1$ & $30.3 \pm 2.8$ & $47.0 \pm 0.5$ & $37.3 \pm 1.8$ \\
NF-LAVL     & $23.4 \pm 1.0$ & $31.8 \pm 7.9$ & $46.1 \pm 3.3$ & $41.9 \pm 2.1$ \\
NFTR (ours) & \sotabest{53.8 \pm 5.9} & \sotabest{41.7 \pm 4.5} & \sotabest{52.0 \pm 0.6} & \sotabest{44.0 \pm 5.5} \\
\bottomrule
\end{tabular}
\end{table}

\textbf{Analysis.} Each baseline is strongest on one maze family and gives the gain back on the other. CGCIVL leads the two PointMaze columns among the value-representation baselines and sits $4.1$ and $0.6$ points below HIQL on AntMaze. LAVL leads the two AntMaze columns and sits $3.7$ points below HIQL on \texttt{pm-tel-stitch}. NFTR is ahead in all four, by $13.7$ and $3.1$ points over CGCIVL and by $5.0$ and $6.7$ points over LAVL on the family where each is strongest. We do not claim the $3.1$ on its own because the intervals overlap.

The composed arm cuts both ways. NF-LAVL improves on LAVL by $+2.1$, $+1.5$ and $+4.6$ points, loses $0.9$ points on \texttt{ant-tel-nav}, and stays $2.1$ to $30.4$ points below NFTR. A stronger value representation does not make the flow redundant, since LAVL alone leads no column, and an expressive policy on a stronger value backbone does not reproduce NFTR either. On \texttt{pm-tel-nav}, B1 exceeds the ablation baseline B0 by $37.1$ points, while NF-LAVL exceeds LAVL by $2.1$ points. AWR reweights the subgoals proposed by each policy, so value learning and policy expressivity act at different stages.

Two structural reasons complete the answer. First, a value scores a pair and states how good $w$ is, whereas the slack scores a triple and states whether $s \to w \to g$ is internally consistent, so no improvement to a pairwise value produces this quantity on its own. Second, the value cannot supply the residual. One could form $\Delta_V = -\log[V(s,w)V(w,g)/V(s,g)]$, but $V_\theta$ is fit by expectile regression on an action-free surrogate under no triangle constraint, so $\Delta_V$ takes both signs and $\exp(-\kappa\Delta_V)$ would reward a negative residual, which is exactly the shortcut the penalty should suppress. The MRN head is used because $\Delta \geq 0$ holds architecturally, which is what makes the penalty one-sided and conservative.

The scope of this comparison is the teleport datasets that NFTR targets. On the deterministic navigate and stitch suites both baselines report results we do not match, and we claim no general ranking.

\subsection{Why Normalizing Flows: Comparison with Diffusion and Flow Matching}
\label{app:nf_vs_other}

To address the question of \emph{why} we choose Normalizing Flows over other expressive multi-modal generative families (diffusion models, flow matching), we compare NFTR against two natural alternatives that share the same hierarchy and triangle-slack reweighting:

\begin{itemize}
\item \textbf{FMTR} (Flow-Matching subgoal policy with Triangle-slack Reweighting): replaces the conditional NF head with a conditional flow-matching subgoal generator, sampled with a 20-step Euler integrator at inference.
\item \textbf{DMTR} (Diffusion subgoal policy with Triangle-slack Reweighting): replaces the NF head with a conditional diffusion subgoal generator, sampled with 20-step DDIM.
\end{itemize}
Everything else (hierarchy, value learning, RWDR weighting, hyperparameters) is held fixed.

\begin{table}[h]
\caption{\textbf{Expressive high-level heads under an identical hierarchy and weight.} Success rate (\%), mean$\pm$std. Our runs use 8 seeds. HIQL values are from the original paper. FMTR replaces the NF head with a conditional flow-matching generator and DMTR with a conditional diffusion generator sampled by DDIM, while hierarchy, value learning, RWDR weight and hyperparameters are held fixed. Per row, the best result is \sotabest{orange bold} and the second-best is \sotasecond{underlined}.}
\label{tab:nf_vs_other_tab}
\centering
\small
\begin{tabular}{lcccc}
\toprule
\textbf{Task} & \textbf{HIQL} & \textbf{NFTR} & \textbf{FMTR} & \textbf{DMTR} \\
\midrule
\texttt{pointmaze-teleport-navigate} & $18.0 \pm 4.0$ & \sotabest{53.8 \pm 5.9} & \sotasecond{38.5 \pm 1.0} & $37.9 \pm 5.7$ \\
\texttt{pointmaze-teleport-stitch}   & $34.0 \pm 4.0$ & \sotasecond{41.7 \pm 4.5} & \sotabest{45.6 \pm 0.9} & $40.1 \pm 6.3$ \\
\texttt{pointmaze-medium-stitch}     & $74.0 \pm 6.0$ & \sotabest{98.1 \pm 1.2} & $95.2 \pm 0.8$ & \sotasecond{96.2 \pm 2.6} \\
\texttt{antmaze-teleport-navigate}   & $42.0 \pm 3.0$ & \sotabest{52.0 \pm 0.6} & \sotasecond{44.1 \pm 3.0} & $42.6 \pm 4.6$ \\
\texttt{antmaze-teleport-stitch}     & $36.0 \pm 2.0$ & \sotabest{44.0 \pm 5.5} & \sotasecond{40.4 \pm 0.9} & $36.2 \pm 2.8$ \\
\texttt{cube-single-play}            & $15.0 \pm 3.0$ & \sotabest{47.1 \pm 9.9} & 
\sotasecond{35.6 \pm 8.6} & $17.4 \pm 4.5$ \\
\texttt{scene-noisy}                 & $25.0 \pm 4.0$ & \sotasecond{32.6 \pm 5.8} & \sotabest{42.9 \pm 1.0} & $31.4 \pm 1.7$ \\
\bottomrule
\end{tabular}
\end{table}

\textbf{Analysis.} \cref{tab:nf_vs_other_tab} places NFTR ahead on five of the seven completed datasets, including margins of $15.3$ points on \texttt{pointmaze-teleport-navigate} and $7.9$ points on \texttt{antmaze-teleport-navigate}. FMTR leads on \texttt{pointmaze-teleport-stitch} by $3.9$ points and on \texttt{scene-noisy} by $10.3$ points. All three expressive heads improve on HIQL on every completed dataset.

FMTR and DMTR are trained by weighted regression rather than weighted maximum likelihood, so their training does not require exact density evaluation and (R2) does not apply to these arms. A one-step generator can also satisfy (R3), although the implementations compared here use iterative sampling. These alternatives show that expressive high-level subgoal modeling admits another extraction route; they do not alter the Gaussian mode-averaging diagnosis of \cref{thm:mode_averaging}, which also holds under weighted least-squares regression by \cref{rem:regression_invariance}. RWDR acts on the sample weight and can be combined with either extraction route.

What remains for the flow are reasons rather than necessity. The first is the empirical margin above. The second is that the exact density is used directly in \cref{fig:loglik_heatmap_task1}, and the NF head adds no integrator and no step-count hyperparameter, whereas FMTR uses a $20$-step Euler integrator and DMTR a $20$-step DDIM schedule. The third is that \cref{thm:rwdr} and \cref{prop:monotonic} are formulated for weighted maximum likelihood. The population targets coincide, so the remaining margin can only reflect approximation, optimization or finite-sample effects, and we record it as an open question.

\subsection{How the Distance Enters the High-Level Weight}
\label{app:weight_rule}

\cref{app:value_repr} compared value and distance signals. We now test how distance enters the high-level weight. All four arms use an NF high-level policy. The value-only arm has no distance head, while the other arms use distance terms in the weight.

\begin{figure}[h]
\centering
\includegraphics[width=0.95\textwidth]{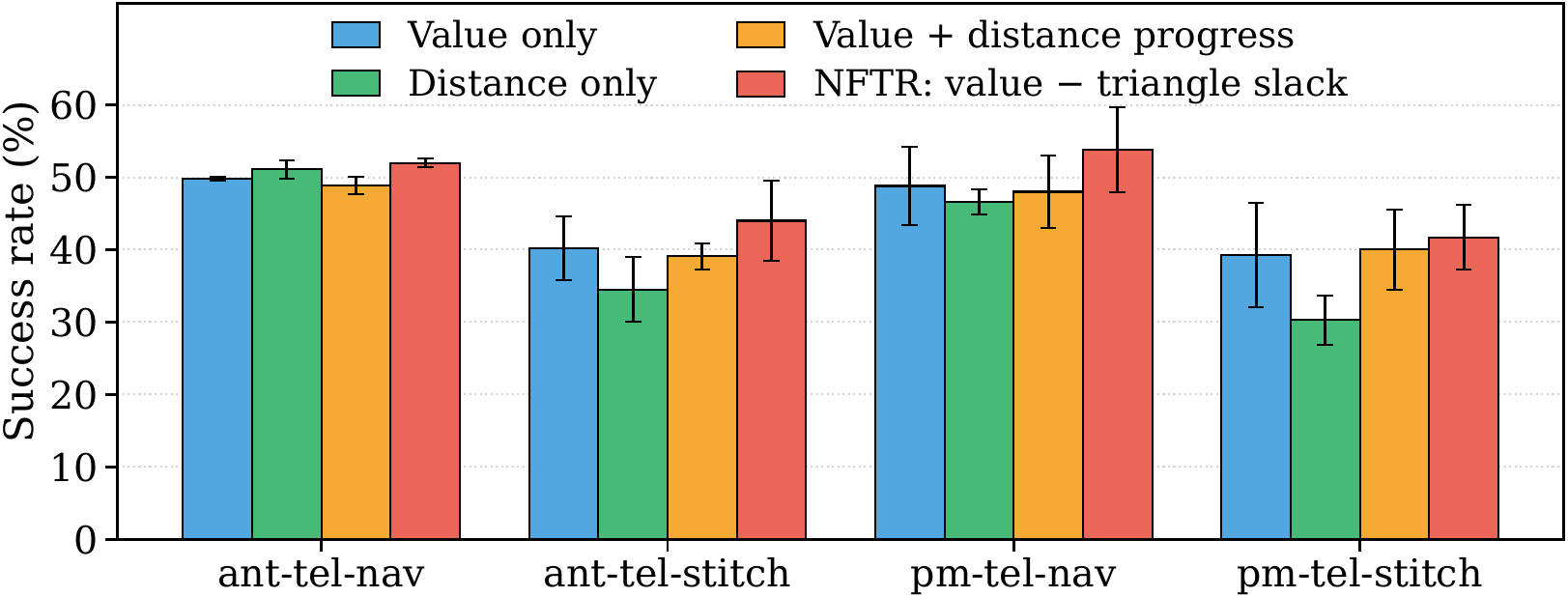}
\caption{\textbf{Four weight rules on the same NF backbone.} Success rate (\%) on the four teleport datasets, 8 seeds, error bars are one standard deviation. \emph{Value only} is the plain AWR weight $\exp(\alpha_H A_{\mathrm{env}})$ and coincides with arm B1 of \cref{tab:factorial}. \emph{Distance only} replaces the advantage by the directional progress $d_\theta(s,g)-d_\theta(w,g)$. \emph{Value $+$ distance progress} adds that progress term to the advantage. \emph{NFTR} subtracts the triangle slack from the advantage. Averaged over the four datasets the arms reach $44.5$, $40.6$, $44.0$ and $47.9$, which the figure does not display.}
\label{fig:weight_rule}
\end{figure}


\textbf{Analysis.} Two comparisons carry the argument. \emph{Distance only} has a lower four-task mean than \emph{Value only}, at $40.6$ versus $44.5$. \emph{Value $+$ distance progress} against NFTR is the matched comparison, since both arms use a value and a distance and differ only in how the distance enters. A progress term is maximized by whatever reduces $d_\theta(\cdot,g)$ fastest, and a teleporter does exactly that, so the term rewards the shortcut and gains nothing over the value alone, at $44.0$ against $44.5$. Before batch normalization, positive slack can only reduce the raw weight. NFTR reaches a four-task mean of $47.9$. The essential point is that the sign constraint, rather than the distance itself, is what makes the geometric signal useful.

\subsection{Direct Evidence That High-Slack Subgoals Are Unreliable}
\label{app:slack_evidence}

On NFTR models and four teleport datasets, we take every transition $s_t \to w = s_{t+25}$, reset the simulator to $s_t$ and replay the identical action sequence. The dispersion of the replayed endpoints then measures non-repeatability on its own terms. No teleporter location and no learned label enters this measurement.

\begin{figure}[h]
\centering
\includegraphics[width=\textwidth]{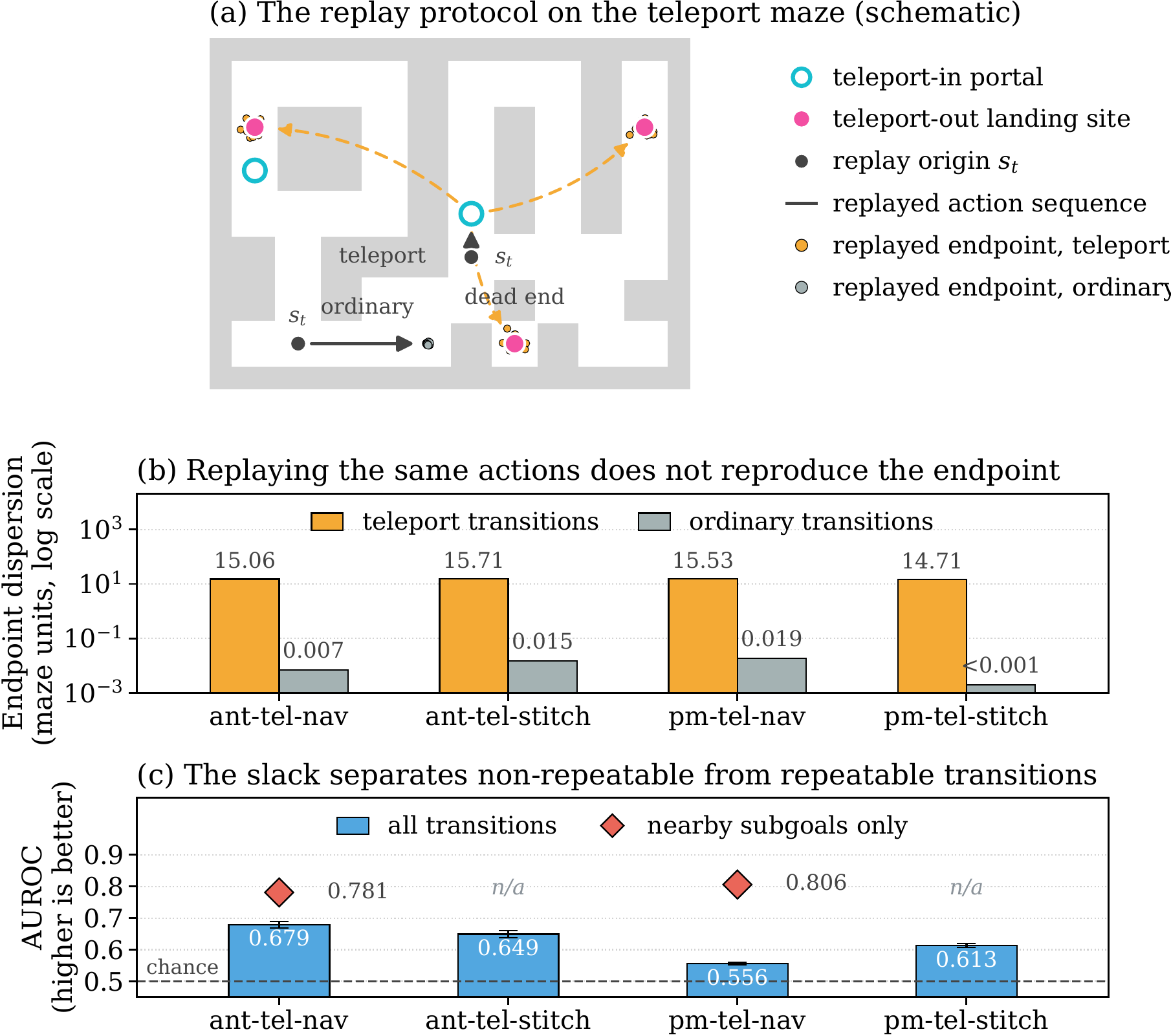}
\caption{\textbf{Replay evidence that high-slack subgoals are non-repeatable.} \textbf{(a)} Schematic of the protocol on the \texttt{teleport} layout, drawn cell by cell from \cref{fig:loglik_heatmap_task1} and using its marker convention. An ordinary transition returns the same endpoint every time. A transition that passes a teleport-in portal is dispatched to one of three teleport-out sites at random, and one of those sites is a dead end, so a single action sequence produces endpoints in three different rooms. \textbf{(b)} Measured endpoint dispersion on a logarithmic axis, in the coordinate units of the environment. Teleport transitions scatter by roughly $15$ units and ordinary ones stay below $0.02$, and the value for \texttt{pm-tel-stitch} falls below the axis floor and is annotated instead. \textbf{(c)} AUROC with which the slack ranks a non-repeatable transition above a repeatable one, over $8$ seeds with $\pm 1$ standard deviation, and the dashed line marks chance. Blue bars use all transitions. Red diamonds restrict the pool to apparently nearby subgoals, where a lucky shortcut is the actual failure mode, and that pool is defined only on the two navigate datasets, so the other two entries are marked as not available. Both panels use the co-trained head.}
\label{fig:replay_evidence}
\end{figure}

\textbf{Analysis.} Teleport transitions are genuinely non-repeatable, since identical actions scatter their endpoints by roughly $15$ units while ordinary transitions stay near zero. The slack ranks a non-repeatable transition above a repeatable one on all four datasets. The association also concentrates where it matters. The highest-slack decile is the most teleport-dense in every dataset, at $1.77$ to $2.70$ times the base rate, and restricted to apparently nearby subgoals AUROC reaches $0.781$ and $0.806$ on the two navigate datasets. The unconditional value is lower because most transitions are on-path regardless. RWDR converts this signal into a lower weight through $\exp(-\kappa\min(\Delta,\Delta_{\max}))$, and where the term is active a teleport subgoal receives roughly $31$ to $34$ percent less weight than a normal one at matched advantage. The essential point is that the association is measured on the environment rather than inferred from the value function, so it does not inherit the bias it is meant to correct.

\subsection{Puzzle Manipulation Results}
\label{app:puzzle}

We evaluate NFTR on the OGBench Puzzle suite, which involves discrete grasp-and-press primitives that produce inherently multi-modal subgoal distributions.

\begin{table}[h]
\centering
\caption{\textbf{OGBench Puzzle Results.} Success rate (\%) on Puzzle $3{\times}3$, $4{\times}4$, $4{\times}5$, and $4{\times}6$, in both \texttt{play} and \texttt{noisy} settings. NFTR uses 8 seeds, mean$\pm$std. HIQL values are from the original paper. Per row, the best result is \sotabest{orange bold} and the second-best is \sotasecond{underlined}.}
\label{tab:puzzle}
\small
\renewcommand{\arraystretch}{0.95}
\resizebox{0.4\textwidth}{!}{
\begin{tabular}{lcc}
\toprule
\textbf{Task} & \textbf{HIQL} & \textbf{NFTR} \\
\midrule
\texttt{puzzle-3x3-play}   & \sotasecond{12 \pm 2}        & \sotabest{14.9 \pm 2.2} \\
\texttt{puzzle-4x4-play}   & \sotasecond{7 \pm 2}         & \sotabest{38.1 \pm 6.9} \\
\texttt{puzzle-4x5-play}   & \sotasecond{4 \pm 1}         & \sotabest{6.7 \pm 1.0}  \\
\texttt{puzzle-4x6-play}   & \sotasecond{3 \pm 1}         & \sotabest{3.5 \pm 0.2}  \\
\midrule
\texttt{puzzle-3x3-noisy}  & \sotabest{51 \pm 11}         & \sotasecond{18.0 \pm 2.3} \\
\texttt{puzzle-4x4-noisy}  & \sotasecond{16 \pm 4}        & \sotabest{53.0 \pm 6.7} \\
\texttt{puzzle-4x5-noisy}  & \sotasecond{5 \pm 1}         & \sotabest{8.8 \pm 0.8}  \\
\texttt{puzzle-4x6-noisy}  & \sotasecond{2 \pm 1}         & \sotabest{5.2 \pm 2.4}  \\
\bottomrule
\end{tabular}}
\end{table}

\textbf{Analysis.} \cref{tab:puzzle} shows that NFTR achieves substantial gains over HIQL on the medium-grid configurations (\texttt{4x4-play}: $+31.1$\,pp, \texttt{4x4-noisy}: $+37.0$\,pp) where the multi-modal structure of valid grasp-and-press primitives matches the Normalizing Flow's modeling capacity, and where observation noise allows triangle-slack to filter spuriously-rewarding subgoals. As grid size grows to 4$\times$5 and 4$\times$6, the bottleneck shifts from per-step subgoal quality to long-horizon \emph{combinatorial planning} (the number of feasible block-press sequences explodes), and the gains compress for both methods. The regression on \texttt{3x3-noisy} is substantial: NFTR scores $18.0$ against HIQL's $51.0$, a $33.0$\,pp decrease. These results alone do not establish the cause of that decrease. Overall, the results suggest that NFTR helps most when (i) multi-modality is present and (ii) noise/stochasticity is non-trivial, while remaining orthogonal to long-horizon planning improvements.

\subsection{HumanoidMaze Results: A Negative Result}
\label{app:humanoidmaze}

We honestly report HumanoidMaze results, where NFTR underperforms HIQL.

\begin{table}[h]
\centering
\caption{\textbf{OGBench HumanoidMaze Results.} Success rate (\%) on HumanoidMaze \texttt{medium}/\texttt{large} navigate and stitch settings (21-DoF humanoid, contact-rich dynamics). NFTR uses 8 seeds, mean$\pm$std. HIQL values are from the original paper. Per row, the best result is \sotabest{orange bold} and the second-best is \sotasecond{underlined}.}
\label{tab:humanoidmaze}
\small
\renewcommand{\arraystretch}{0.95}
\resizebox{0.5\textwidth}{!}{
\begin{tabular}{lcc}
\toprule
\textbf{Task} & \textbf{HIQL} & \textbf{NFTR} \\
\midrule
\texttt{humanoidmaze-medium-navigate} & \sotabest{89 \pm 2} & \sotasecond{49.5 \pm 3.8} \\
\texttt{humanoidmaze-large-navigate}  & \sotabest{49 \pm 4} & \sotasecond{7.7 \pm 1.5}  \\
\texttt{humanoidmaze-medium-stitch}   & \sotabest{88 \pm 2} & \sotasecond{62.9 \pm 7.5} \\
\texttt{humanoidmaze-large-stitch}    & \sotabest{28 \pm 3} & \sotasecond{7.8 \pm 2.3}  \\
\bottomrule
\end{tabular}}
\end{table}

\textbf{Analysis.} \cref{tab:humanoidmaze} reports a consistent regression of NFTR relative to HIQL across all four HumanoidMaze settings (e.g., $-39.5$\,pp on \texttt{medium-navigate}, $-41.3$\,pp on \texttt{large-navigate}). This regression has two coupled causes:
\begin{enumerate}
\item \textbf{Physical realizability of multi-modal subgoals.} The 21-DoF humanoid's contact-rich dynamics make many of the multiple subgoal modes generated by the Normalizing Flow not all physically realizable: a ``valid corridor'' entrance in maze coordinates may be unreachable for a humanoid given balance and footstep constraints. Consequently, multi-modal subgoal proposals introduce noise into AWR training, which a unimodal Gaussian (HIQL) avoids by collapsing to the single safest mode.
\item \textbf{Limited environmental stochasticity.} HumanoidMaze does not contain teleporters or strong random transitions, so the optimistic-bias issue that triangle-slack is designed to address is largely absent. RWDR therefore provides limited gain while still introducing additional optimization complexity.
\end{enumerate}
This is consistent with our scope statement (\cref{sec:experiments}). NFTR targets mode collapse and optimistic bias, and gains can arise when either is prominent. In high-DoF, contact-rich, deterministic regimes, neither appears to dominate. The additional NF and RWDR machinery does not overcome the long-horizon control difficulty here.

\subsection{Comprehensive RWDR Coefficient $\kappa$ Ablation}
\label{app:kappa_full}

We report the full $\kappa$ sweep across the 13 environments in \cref{tab:kappa_full}. For these tasks, the NFTR scores reported elsewhere correspond to $\kappa=0$ on \texttt{pm-tel-stitch}, \texttt{pm-med-stitch}, \texttt{antmaze-medium-stitch}, and \texttt{antmaze-large-stitch}. They correspond to $\kappa=1$ on \texttt{pointmaze-large-stitch}, \texttt{ant-tel-nav}, and \texttt{ant-tel-stitch}, and to $\kappa=2$ on \texttt{pm-tel-nav} and the four Cube and Scene tasks. The \texttt{antmaze-teleport-explore} result uses $\kappa=5$. The values in \cref{tab:nftr-hyperparams} are reference settings, not a universal coefficient.

\begin{table*}[h]
\centering
\caption{\textbf{Full $\kappa$ Ablation.} Sensitivity to the RWDR penalty coefficient $\kappa\in\{0,0.5,1.0,2.0,5.0\}$ across the 13 environments included in this sweep. Success rate (\%), 8 seeds, mean$\pm$std. Per row (i.e., per environment), the best $\kappa$ is \sotabest{orange bold} and the second-best is \sotasecond{underlined}.}
\label{tab:kappa_full}
\small
\renewcommand{\arraystretch}{0.95}
\resizebox{0.9\textwidth}{!}{
\begin{tabular}{lccccc}
\toprule
\textbf{Task} & $\kappa{=}0.0$ & $\kappa{=}0.5$ & $\kappa{=}1.0$ & $\kappa{=}2.0$ & $\kappa{=}5.0$ \\
\midrule
\texttt{pointmaze-teleport-navigate} & $48.2 \pm 3.6$ & \sotasecond{48.7 \pm 4.5} & $48.6 \pm 10.5$ & \sotabest{53.8 \pm 5.9} & $47.3 \pm 2.5$ \\
\texttt{pointmaze-teleport-stitch}   & \sotabest{41.7 \pm 4.5} & $37.1 \pm 9.1$ & \sotasecond{40.8 \pm 8.8} & $40.6 \pm 7.9$ & $40.0 \pm 8.9$ \\
\texttt{pointmaze-medium-stitch}     & \sotabest{98.1 \pm 1.2} & \sotasecond{98.0 \pm 0.2} & $95.3 \pm 2.9$ & $86.4 \pm 8.7$ & $87.9 \pm 5.5$ \\
\texttt{pointmaze-large-stitch}      & \sotasecond{30.3 \pm 6.3} & $29.6 \pm 5.6$ & $30.0 \pm 5.7$ & \sotabest{31.8 \pm 5.9} & $28.8 \pm 4.9$ \\
\texttt{antmaze-teleport-navigate}   & $49.9 \pm 0.3$ & \sotasecond{50.5 \pm 0.3} & \sotabest{52.0 \pm 0.6} & $48.0 \pm 3.1$ & $41.9 \pm 3.2$ \\
\texttt{antmaze-teleport-stitch}     & $41.1 \pm 3.0$ & $40.5 \pm 3.7$ & \sotabest{44.0 \pm 5.5} & \sotasecond{41.6 \pm 1.6} & $33.1 \pm 3.4$ \\
\texttt{antmaze-teleport-explore}    & $14.2 \pm 2.7$ & $17.6 \pm 4.5$ & \sotasecond{19.8 \pm 3.1} & $19.0 \pm 1.4$ & \sotabest{26.1 \pm 1.4} \\
\texttt{antmaze-medium-stitch}       & \sotabest{81.0 \pm 3.7} & \sotasecond{80.8 \pm 4.3} & $76.0 \pm 5.5$ & $70.2 \pm 8.5$ & $31.0 \pm 6.2$ \\
\texttt{antmaze-large-stitch}        & \sotabest{16.9 \pm 2.1} & $13.6 \pm 2.5$ & $14.4 \pm 2.4$ & \sotasecond{15.9 \pm 0.8} & $11.7 \pm 2.1$ \\
\texttt{cube-single-play}            & $22.2 \pm 6.0$ & $24.7 \pm 4.3$ & $31.3 \pm 10.0$ & \sotabest{47.1 \pm 9.9} & \sotasecond{46.2 \pm 4.4} \\
\texttt{cube-single-noisy}           & $55.2 \pm 3.3$ & $62.8 \pm 5.1$ & \sotasecond{64.0 \pm 2.7} & \sotabest{67.4 \pm 0.9} & $59.4 \pm 3.9$ \\
\texttt{scene-play}                  & \sotasecond{51.7 \pm 3.7} & \sotabest{52.2 \pm 1.1} & $47.7 \pm 8.7$ & $45.6 \pm 3.8$ & $24.8 \pm 5.0$ \\
\texttt{scene-noisy}                 & $27.1 \pm 2.1$ & \sotasecond{32.8 \pm 3.1} & \sotabest{34.8 \pm 2.4} & $32.6 \pm 5.8$ & $22.2 \pm 4.7$ \\
\bottomrule
\end{tabular}}
\end{table*}

\textbf{Three regimes emerge from \cref{tab:kappa_full}.}

\textit{(1) Stitching tasks.} The highest means on \texttt{pointmaze-medium-stitch} and \texttt{antmaze-medium-stitch} occur at $\kappa=0$ in \cref{tab:kappa_full}. On the former task, this configuration scores $98.1\%$ versus HIQL's reported $74\%$. A positive slack penalty provides no higher mean on these two tasks.

\textit{(2) Teleport tasks.} The highest mean occurs at $\kappa=0$ for \texttt{pm-tel-stitch}, at $1$ for the two AntMaze teleport tasks, at $2$ for \texttt{pm-tel-nav}, and at $5$ for \texttt{antmaze-teleport-explore}. The preferred coefficient varies by task.

\textit{(3) Manipulation tasks.} The highest means occur at $\kappa=2$ for both Cube tasks, at $0.5$ for \texttt{scene-play}, and at $1$ for \texttt{scene-noisy}. The preferred coefficient varies across these tasks.

The sweep shows that the preferred $\kappa$ varies across tasks, with both gains and regressions as the penalty increases. It therefore describes the scope and sensitivity of RWDR rather than establishing a universal coefficient or a monotonic relation with environmental stochasticity.

\subsection{Comparison with NF-RLBC Variants}
\label{app:nf_rlbc}

For completeness, we compare NFTR against goal-conditioned and hierarchical goal-conditioned variants of NF-RLBC~\citep{ghugare2025normalizing}, denoted \textbf{GC-NF-RLBC} and \textbf{H-GC-NF-RLBC}. These baselines also use Normalizing Flows without our triangle-slack reweighting. This comparison measures overall performance against NF-based alternatives. It does not isolate RWDR because the methods also differ in hierarchy and training objective.

\begin{table*}[h]
\centering
\caption{\textbf{Comparison with NF-RLBC variants.} Success rate (\%), mean$\pm$std. Our runs use 8 seeds. HIQL values are from the original paper. GC-NF-RLBC: goal-conditioned NF-RLBC. H-GC-NF-RLBC: hierarchical goal-conditioned NF-RLBC. Per row, the best result is \sotabest{orange bold} and the second-best is \sotasecond{underlined}.}
\label{tab:nf_rlbc}
\small
\renewcommand{\arraystretch}{0.95}
\resizebox{0.9\textwidth}{!}{
\begin{tabular}{lcccc}
\toprule
\textbf{Task} & \textbf{HIQL} & \textbf{NFTR} & \textbf{GC-NF-RLBC} & \textbf{H-GC-NF-RLBC} \\
\midrule
\texttt{pointmaze-teleport-navigate} & $18 \pm 4$        & \sotabest{53.8 \pm 5.9} & $13.2 \pm 6.4$        & \sotasecond{40.1 \pm 6.5} \\
\texttt{pointmaze-teleport-stitch}   & \sotasecond{34 \pm 4} & \sotabest{41.7 \pm 4.5} & $19.5 \pm 2.8$        & $4.9 \pm 4.8$ \\
\texttt{pointmaze-medium-stitch}     & \sotasecond{74 \pm 6} & \sotabest{98.1 \pm 1.2} & $22.3 \pm 6.5$        & $0.3 \pm 0.4$ \\
\texttt{pointmaze-large-stitch}      & \sotasecond{13 \pm 6} & \sotabest{30.0 \pm 5.7} & $2.0 \pm 2.1$         & $0.8 \pm 1.5$ \\
\texttt{pointmaze-giant-stitch}      & $0 \pm 0$         & $0.0 \pm 0.0$           & $0.0 \pm 0.0$         & $0.0 \pm 0.0$ \\
\texttt{antmaze-teleport-navigate}   & $42 \pm 3$        & \sotabest{52.0 \pm 0.6} & $33.2 \pm 3.2$        & \sotasecond{49.5 \pm 2.4} \\
\texttt{antmaze-teleport-stitch}     & $36 \pm 2$ & \sotabest{44.0 \pm 5.5} & $14.5 \pm 1.3$        & \sotasecond{36.8 \pm 2.0} \\
\texttt{antmaze-teleport-explore}    & \sotabest{34 \pm 15} & \sotasecond{26.1 \pm 1.4} & $0.7 \pm 0.3$         & $0.1 \pm 0.1$ \\
\texttt{antmaze-medium-stitch}       & \sotabest{94 \pm 1}  & \sotasecond{81.0 \pm 3.7} & $29.0 \pm 3.2$        & $61.0 \pm 2.2$ \\
\texttt{antmaze-large-stitch}        & \sotabest{67 \pm 5}  & \sotasecond{16.9 \pm 2.1} & $0.2 \pm 0.1$         & $4.2 \pm 2.3$ \\
\texttt{antmaze-giant-stitch}        & \sotabest{2 \pm 2}   & $0.0 \pm 0.0$           & $0.0 \pm 0.0$         & $0.0 \pm 0.0$ \\
\texttt{cube-single-play}            & $15 \pm 3$ & \sotabest{47.1 \pm 9.9} & \sotasecond{18.7 \pm 0.3}        & $9.1 \pm 2.5$ \\
\texttt{cube-single-noisy}           & \sotasecond{41 \pm 6} & \sotabest{67.4 \pm 0.9} & $14.1 \pm 0.8$        & $18.3 \pm 1.6$ \\
\texttt{scene-play}                  & \sotasecond{38 \pm 3} & \sotabest{45.6 \pm 3.8} & $10.2 \pm 1.2$        & $20.4 \pm 2.6$ \\
\texttt{scene-noisy}                 & \sotasecond{25 \pm 4} & \sotabest{32.6 \pm 5.8} & $4.1 \pm 0.2$         & $11.3 \pm 0.7$ \\
\bottomrule
\end{tabular}}
\end{table*}

\textbf{Analysis.} \cref{tab:nf_rlbc} shows NFTR outperforming both GC-NF-RLBC and H-GC-NF-RLBC on 13 of 15 datasets and tying on the remaining two, where all three methods score zero. The flat goal-conditioned NF policy (GC-NF-RLBC) struggles broadly. The hierarchical variant (H-GC-NF-RLBC) recovers some performance, but NFTR has higher mean success on all 13 non-tied tasks, often by a substantial margin. The gap on \texttt{pointmaze-medium-stitch}, at $98.1$ against $0.3$, is not evidence for the reweighting. \cref{tab:kappa_full} reports $98.1$ on this dataset at $\kappa=0$, so the reweighting cannot be responsible for it. The two methods also differ in hierarchy, value learning, and policy training. This cross-method gap cannot be assigned to one component. On \texttt{antmaze-medium/large/giant-stitch}, HIQL itself outperforms NFTR, consistent with the limits of this comparison on long-horizon tasks.

\subsection{Training Curves and Practical Stability under SGD}
\label{app:training-curves}
 
We record evaluation success every $0.1\times 10^{6}$ gradient steps on
the 15 evaluated tasks. \Cref{fig:training-curves} shows four representative
regimes. Mean success levels off after roughly $0.3$ to $0.5\times 10^{6}$
steps in most of these runs, without a collapse in the plotted interval.
For panel (b), seed-level standard deviation contracts from $\pm 11.9$
at $0.1\times 10^{6}$ steps to $\pm 2.4$ at $1\times 10^{6}$ steps;
panel (c) varies between $\pm 1.4$ and $\pm 5.7$ over training.
These curves describe evaluation behavior. They do not establish
per-step improvement or convergence of the weighted-MLE objective.
 
\begin{figure}[t]
    \centering
    \includegraphics[width=0.92\linewidth]{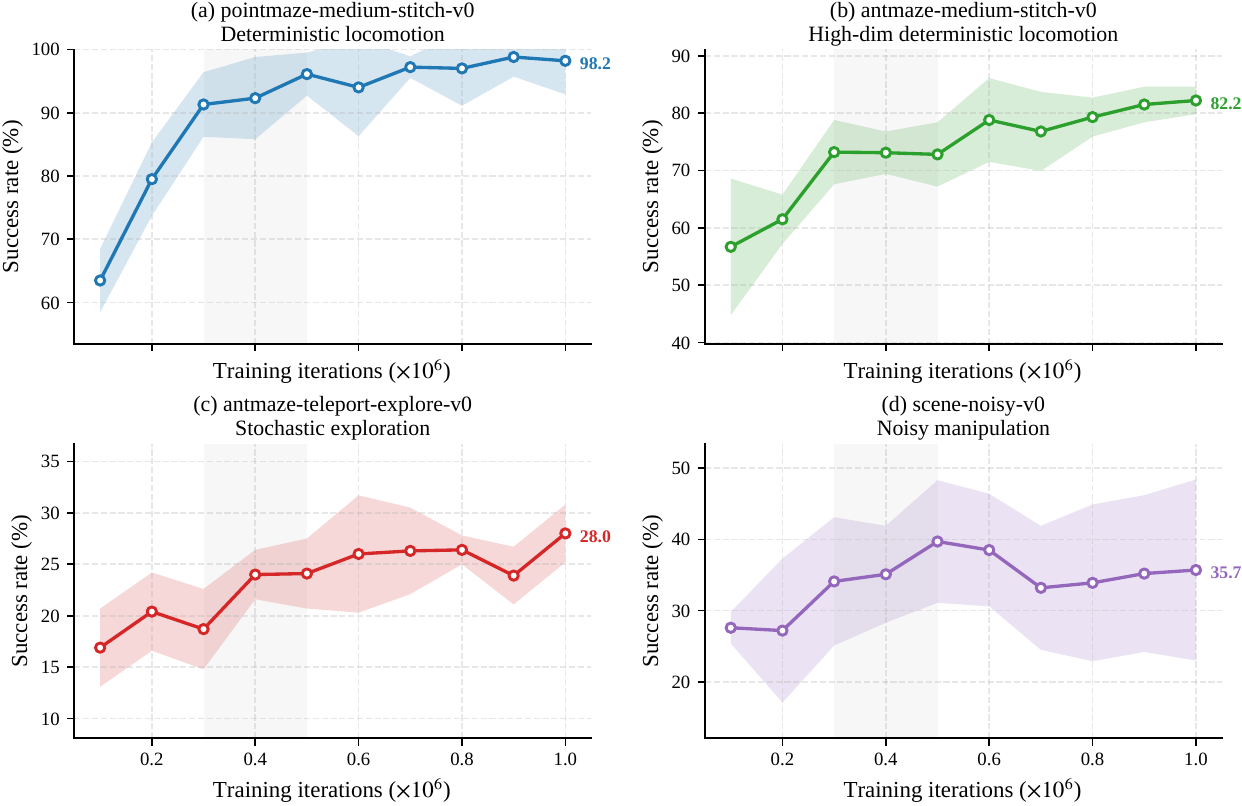}
    \caption{
        \textbf{Training curves of NFTR across four representative
        environments.} Solid curves show mean success rate over $8$
        random seeds. Shaded bands show $\pm 1$ standard deviation.
        The light gray vertical band marks the
        $0.3$ to $0.5\times 10^{6}$ convergence window. The colored
        number on the right of each curve is the mean success rate at
        $1\times 10^{6}$ steps. (a) Deterministic locomotion.
        (b) High-dimensional deterministic locomotion.
        (c) Stochastic locomotion, the regime in which
        \cref{prop:slack}(4) reduces to a one-sided bound and which
        therefore tests the practical stability claim most directly.
        (d) Noisy manipulation.
    }
    \label{fig:training-curves}
\end{figure}

\subsection{Visualization of Learned Representations}

Both heatmaps in \cref{fig:loglik_heatmap_task1} visualize the high-level subgoal preference for a fixed state--goal pair $(s,g)$ (green/red dots). For each candidate waypoint $w$, we compute a target subgoal representation $z_{\text{tgt}}=\phi(s,w)$ and plot the relative density $\log p_\theta(z_{\text{tgt}}\mid s,g)$ over feasible maze regions (brighter $\Rightarrow$ higher preference within each panel). Gray blocks indicate walls / unreachable space. In the teleport maze, cyan hollow circles denote \emph{teleport-in} portals (trigger locations), while magenta filled circles denote \emph{teleport-out} landing sites.

In \cref{fig:loglik_heatmap_task1}\,(a), HIQL's unimodal Gaussian produces a diffuse landscape that effectively averages across alternative routes, spreading preference beyond specific corridor structures. In contrast, \cref{fig:loglik_heatmap_task1}\,(b) shows NFTR's Normalizing Flow yielding a more structured, multi-modal preference pattern, concentrating probability mass around distinct feasible regions, which better reflects multiple valid subgoal modes.

\begin{figure}[!htbp]
\centering
\begin{subfigure}[t]{0.49\columnwidth}
  \centering
  \includegraphics[width=\linewidth]{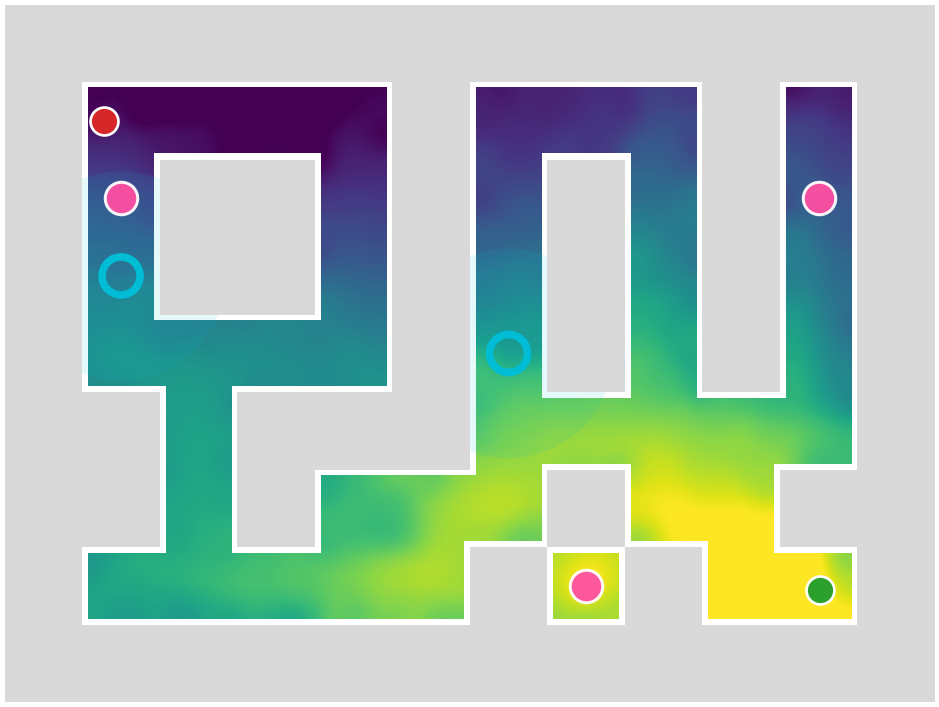}
  \caption{HIQL}
  \label{fig:hiql_vis}
\end{subfigure}\hfill
\begin{subfigure}[t]{0.49\columnwidth}
  \centering
  \includegraphics[width=\linewidth]{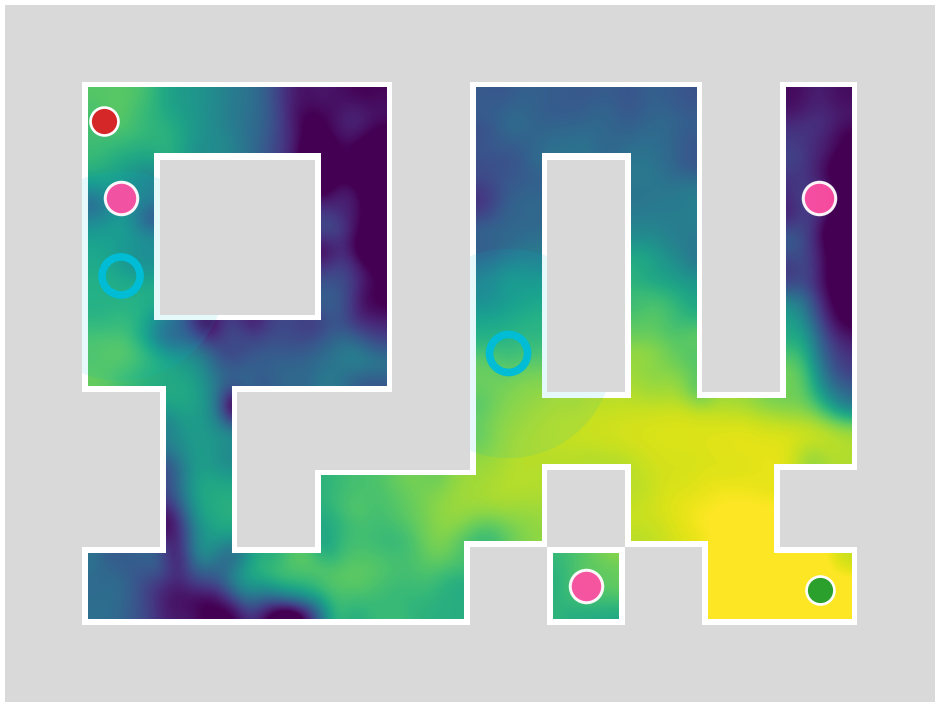}
  \caption{NFTR}
  \label{fig:nftr_vis}
\end{subfigure}
\caption{Log-likelihood heatmap on antmaze-teleport-navigate \texttt{task1}.}
\label{fig:loglik_heatmap_task1}
\end{figure}